\documentclass[pdflatex,sn-mathphys-num]{sn-jnl}

\usepackage{graphicx}%
\usepackage{multirow}%
\usepackage{amsmath,amssymb,amsfonts}%
\usepackage{amsthm}%
\usepackage{mathrsfs}%
\usepackage[title]{appendix}%
\usepackage{xcolor}%
\usepackage{textcomp}%
\usepackage{manyfoot}%
\usepackage{booktabs}%
\usepackage{algorithm}%
\usepackage{algorithmicx}%
\usepackage{algpseudocode}%
\usepackage{listings}%
\newcommand{\E}{\mathbb{E}}
\newcommand{\p}{\mathbb{P}}
\newtheorem{assumption}{Assumption}

\theoremstyle{thmstyleone}%
\newtheorem{theorem}{Theorem}
\newtheorem{proposition}[theorem]{Proposition}%

\theoremstyle{thmstyletwo}%
\newtheorem{remark}{Remark}%

\theoremstyle{thmstylethree}%
\newtheorem{definition}{Definition}%

\begin{document}

\title[Article Title]{Entropy Regularized Reinforcement Learning for Zero-Sum Stochastic Differential Games in a Regime-Switching Jump-Diffusion Process}


\author[1]{\fnm{Congde} \sur{Hu}}\email{hcd.math@gmail.com}

\author[2]{\fnm{Zhuo} \sur{Jin}}\email{zhuo.jin@mq.edu.au}

\author[3]{\fnm{Danping} \sur{Li}}\email{dpli@fem.ecnu.edu.cn}

\author*[1]{\fnm{Lin} \sur{Xu}}\email{xulinahnu@gmail.com}

\affil*[1]{\orgdiv{School of Mathematics and Statistics}, \orgname{Anhui Normal University}, \orgaddress{\city{Wuhu}, \postcode{241002}, \state{Anhui}, \country{China}}}

\affil[2]{\orgdiv{Department of Actuarial Studies and Business Analytics}, \orgname{Macquarie University}, \orgaddress{\city{Sydney}, \postcode{2109}, \state{NSW}, \country{Australia}}}

\affil[3]{\orgdiv{School of Statistics, KLATASDS-MOE, China Inclusive Ageing Finance Research Center}, \orgname{East China Normal University}, \orgaddress{\city{Shanghai}, \postcode{200062}, \state{Shanghai}, \country{China}}}


\abstract{To address parameter misspecification and sudden structural environmental changes in conventional stochastic differential game (SDG) frameworks, this paper introduces a distributional control approach that characterizes optimal strategies as probability distributions over actions, conditioned on the continuous state, the discrete regime state, and parameters. This forms a reinforcement learning framework for entropy-regularized zero-sum stochastic differential games (ERRL-ZSSDGs) in a regime-switching jump-diffusion process. Using the dynamic programming principle (DPP), we derive the associated coupled systems of Hamilton-Jacobi-Bellman-Isaacs (HJBI) equations, from which equilibrium strategies are expressed via gradients of the value function. For linear-quadratic problems, semi-analytical solutions for both value function and equilibrium strategies are obtained by solving a system of coupled ordinary differential equations (ODEs). In more general settings, an Actor-Critic policy improvement algorithm is developed to approximate the value functions and equilibrium policies across different regimes. The method is applied to an investment game, and numerical examples illustrate the effect of the temperature parameter and regime transitions on optimal policies and values.}




\maketitle
\section{Introduction}\label{sec1}
Stochastic differential games (SDGs) provide a mathematical framework for modeling decision-making processes involving multiple agents in dynamic environments influenced by randomness and structural shifts. These games are typically formulated through stochastic differential equations (SDEs), which describe the evolution of the system state over time under the combined effects of deterministic and stochastic forces. SDGs are usually used in scenarios where agents have conflicting or cooperative objectives and the environment introduces uncertainty. The past decades have witnessed a rapidly expanding literature on the research of the SDGs, such as modeling construction and theoretical analysis (\cite{moon2018Risk}, \cite{lv2023linear}) , numerical methods (\cite{Kushner2004numerical}, \cite{song2008numerical}) and applications in industry (\cite{yeung2008cooperative}, \cite{song2023decision}), insurance (\cite{jin2013optimal}, \cite{wang2021stochastic}), economics and management (\cite{dockner2000differential}, \cite{savku2022stochastic}).

Zero-sum stochastic differential games (ZSSDGs) are a specialized class of SDGs in which the total payoffs of all players sum to zero, meaning any gain by one player is exactly balanced by an equivalent loss of another. This framework represents scenarios where players' interests are strictly adversarial, and the overall benefit or loss remains constant. In recent years, several papers have focused on theoretical research on ZSSDGs. For example, \cite{li2012saddle} establishes sufficient conditions for the existence of a saddle point in a time-dependent discrete Markov zero-sum game up to a specified stopping time. \cite{lv2020two} demonstrates the equivalence between the value function of a ZSSDG and the viscosity solution to the associated Hamilton-Jacobi-Bellman-Isaacs (HJBI) equation. \cite{lv2024solving} studies a class of zero-sum stopping games in a regime-switching model. Beyond these theoretical advances, given the broad applicability of ZSSDGs in diverse fields, research on ZSSDGs has been highly active and productive. For instance, in risk management, \cite{patil2023risk} considers risk-minimizing ZSSDG via path integral control; \cite{guan2016stochastic} applies the ZSSDG framework to analyze interactions between two defined-contribution (DC) pension plans.

To date, the majority of studies have presumed that the parameters of the controlled system are either constants or deterministic functions of the system. Within the framework of feedback control, optimal strategies depend on the model parameters, the current system state, and the feedback control function. However, in practical applications, these system parameters must be estimated. Due to data inaccuracies, methodological errors, or abrupt macroeconomic structural changes (regime shifts), the validity of these parameters may be compromised. This introduces risks such as estimation errors, failure to adapt to new environmental states, or the possibility of identifying a local optimizer as the global optimizer during the search for optimal strategies.

To address these challenges, we develop a novel framework for determining optimal strategies in SDGs using reinforcement learning (RL). RL inherently balances exploration and exploitation, thereby circumventing local optima and adapting to dynamic environments. This balance allows for a trade-off between exploring new actions and leveraging existing information to maximize rewards. By continuously experimenting with diverse strategies, the algorithm avoids being trapped in local optima. However, a significant challenge in the numerical design of RL agents is their tendency to inadequately explore various states, which can impede the discovery of superior strategies. To counteract this, we aim to distribute the output probabilities of the policy network across multiple actions, assigning non-zero probabilities to enhance exploration. Entropy quantifies the dispersion of this probability distribution: low entropy indicates concentration, while high entropy signifies greater randomness. The decision-maker seeks a high-entropy policy to achieve faster exploration.

This exploratory paradigm has been widely applied in various domains. \cite{wang2020reinforcement} pioneered the application of this approach to stochastic control problems in continuous time and space. Furthermore, \cite{wang2020continuous} and \cite{dai2023learning} consider the mean-variance problem in a reinforcement learning framework. \cite{firoozi2022exploratory} and \cite{guo2022entropy} discuss entropy regularization for mean-field games. \cite{hao2022entropy}, \cite{sun2023reinforcement} introduce the entropy-regularized RL framework into the study of SDGs and \citet{Huang2025convergence} analyzes the covergence of the policy iterations for entropy regularized control problems.

However, the optimal control problems discussed previously are predominantly driven by pure Brownian motion and have been extensively examined through the lens of RL. Recently, \cite{gao2024reinforcement} investigates an RL model with jump-diffusion processes, providing optimal exploratory solutions for the linear quadratic regulator (LQR) problem. This paper integrates recent advancements in RL and constructs a ZSSDG within the RL framework under a regime-switching jump-diffusion model. By incorporating a continuous-time Markov chain alongside the jump-diffusion process, our model simultaneously captures both discrete tail events (jumps) and persistent macroeconomic or environmental structural changes (regimes). In fact, 

There are three main contributions in this paper. First, this paper advances the theoretical and numerical understanding of ZSSDGs by integrating RL into a regime-switching jump-diffusion framework. Unlike traditional approaches that assume constant or simply state-dependent parameters, our work addresses inherent uncertainties, estimation errors, and structural shifts by introducing a novel distributional control mechanism. This mechanism ensures that the optimal strategies for both players are characterized by probability distributions that depend on the continuous state, the current discrete regime, and the system parameters, rather than deterministic functions.

Second, we establish an exploratory regime-switching ZSSDG framework for participants from both sides and apply it to two problems. This framework not only enhances the exploration-exploitation trade-off in RL but also mitigates the risk of converging to local optima under varying market regimes. 

Third, although regime-switching effects have been increasingly considered in various control and game-theoretic settings \cite{song2008numerical,Kushner2004numerical,gruen2012probabilistic,Huang2025convergence,feng2021optimal,Nguyen2021general,chen2022stochastic,gutierrez2024markovian}, state-dependent structural shifts are ubiquitous in real-world economic, financial, and engineering dynamical systems \cite{jin2013optimal,lv2024solving}. Regime-switching causes discrete transitions in the drift, diffusion, and jump intensity of the system according to external environments or internal states, essentially introducing extra stochastic structure and mode-dependent characteristics. Within the jump-diffusion and exploratory RL framework, regime-switching further complicates the model: on the one hand, the switching mechanism itself is stochastic and time-varying, making it incompatible with conventional parameter estimation and policy optimization methods; on the other hand, the three types of uncertainty, namely regime-switching, jump behavior, and randomized exploration, are mutually coupled, which breaks the logic of moment averaging and policy exploration construction in classical diffusion models and directly invalidates the original exploratory dynamics modeling approach.

The structure of this paper is organized as follows. Section~\ref{Section 2} constructs a theoretical framework for ZSSDGs under the regime-switching jump-diffusion model incorporating a learning process. Section~\ref{Section 3} applies the proposed framework to the regime-switching linear-quadratic ZSSDG problem. Section~\ref{Section 4} designs an actor-critic RL algorithm to solve the ZSSDG problem in general cases. Section~\ref{Section 5} applies this algorithm to a regime-switching investment game. Finally, Section~\ref{Section 6} concludes the paper.

\textbf{Notations:}
$\mathbb R^+$ denotes the positive reals, and
$\mathbb R_K^d:=\{x\in\mathbb R^d:\lvert x\rvert\le K\}$. Let $\mathbb R^{k\times\ell}$ be the space of
real $k\times\ell$ matrices. For a matrix $A$, $A^\top$ is its transpose, $\lvert A\rvert$ its Frobenius norm,
$A^2=AA^\top$, $A\circ B=\mathrm{tr}(AB^\top)$, and $A^{1/2}$ the square root of a positive semidefinite matrix.
For a function $f$, $f_x$ and $f_{xx}$ denote the first- and second-order partial derivatives. The symbol $\mathbb E_{t,x,i}[\cdot]$ denotes the conditional expectation given $X_t=x$, $\theta_t=i$. For $a,b\in\mathbb R$,
$a\wedge b=\min\{a,b\}$. The symbol $\delta$ denotes the variational operator. Finally, $U(B)$ denotes the uniform distribution on $B$, $\mathcal N(\mu,\Sigma)$ a Gaussian
distribution, and $\Phi$ the cdf of $\mathcal N(0,1)$.
\section{Formulation and problem}\label{Section 2}
This section develops the mathematical framework for a ZSSDG within a regime-switching jump-diffusion model incorporating learning. To facilitate comparison between the cases with and without learning, we adopt as a benchmark the jump-diffusion ZSSDG, introduced by \cite{biswas2012zero}, outlined in Subsection \ref{sub21}. Subsection \ref{sub22} then presents the construction of an ERRL framework for pure diffusion-controlled systems, discusses the challenges in extending it to systems with jumps and regime-switching, and proposes methods to address these challenges. Finally, Subsection \ref{sub23} provides the integrated model of ZSSDG with ERRL.
\subsection{Classical ZSSDG within regime-switching jump-diffusion model}\label{sub21}
For a fixed positive constant $T$ and $t\in[0,T]$, let $\left(\Omega_t, \mathcal{F}_t,\{\mathcal{F}_{t,s}\}_{t\leq s \leq T}, \p_t\right)$
be a filtered probability space satisfying the usual conditions. Let $\theta = \{\theta_s, t \leq s \leq T\}$ be a continuous-time Markov chain on this probability space, taking values in a finite state space $\mathcal{S} = \{1, 2, \ldots, m\}$, which represents the regime states. The generator matrix of this Markov chain is defined as $Q = (q_{ij})_{m \times m}$, where $q_{ij} \geq 0$ ($i \neq j$) represents the transition rate from regime $i$ to regime $j$, satisfying $q_{ii} = -\sum_{j=1, j \neq i}^{m} q_{ij}$. The ZSSDG is defined on $\left(\Omega_t, \mathcal{F}_t,\{\mathcal{F}_{t,s}\}_{t\leq s \leq T}, \p_t\right)$, and consists of the following controlled SDE
\begin{equation}\label{CLSDE}
\left\{\begin{array}{ll}
d X_s=b(s, X_{s}, \theta_{s} ; y_s, z_s) d s+\sigma(s, X_{s}, \theta_{s} ; y_s, z_s) d W_s\\
\qquad\quad+\displaystyle\int_{E} \eta(s, X_{s{-}}, \theta_{s{-}} ; y_s, z_s ; w) \tilde{N}(d s, d w),\, s \in(t, T],\\
X_t=x, \quad \theta_t=i,
\end{array} \right.
\end{equation}
as well as the payoff function
\begin{equation}\label{CLpayoff function}
J^{cl}(t, x, i ; y, z)
=\E_{t,x,i}\left[\int_t^T r(s, X_s, \theta_s ; y_s, z_s) d s+g(X_T, \theta_T)\right] ,
\end{equation}
where
\begin{equation*}
\begin{aligned}
b:&[0,T]\times\mathbb{R}^d\times\mathcal{S}\times\mathcal{Y}\times\mathcal{Z}\rightarrow\mathbb{R}^d,\quad E := \mathbb{R}^\ell \setminus \{0\},\\
\sigma:&[0,T]\times\mathbb{R}^d\times\mathcal{S}\times\mathcal{Y}\times\mathcal{Z}\rightarrow\mathbb{R}^{d \times k},\\
\eta:&[0,T]\times\mathbb{R}^d\times\mathcal{S}\times\mathcal{Y}\times\mathcal{Z}\times\mathbb{R}^\ell\rightarrow\mathbb{R}^{d\times\ell},\\
r:&[0,T]\times\mathbb{R}^d\times\mathcal{S}\times\mathcal{Y}\times\mathcal{Z}\rightarrow\mathbb{R},\quad g:\mathbb{R}^d\times\mathcal{S}\rightarrow\mathbb{R},
\end{aligned}
\end{equation*}
$y$ and $z$ are two predictable control processes with values in compact metric spaces $\mathcal{Y}\subseteq \mathbb{R}^{n1}$ and $\mathcal{Z}\subseteq \mathbb{R}^{n2}$, respectively. $W_s$ is a $k$-dimensional Brownian motion on $\left(\Omega_t, \mathcal{F}_t,\{\mathcal{F}_{t,s}\}_{t\leq s \leq T}, \p_t\right)$. Furthermore, $N(ds,dw):=(N_1(d s, d w_1), \cdots, N_{\ell}(d s, d w_{\ell}))^{\top}$ is a Poisson random measure on $[0, T]\times E$ with intensity measure $\nu(d w):=(\nu_1(d w),\cdots,\nu_\ell(d w))^{\top}$ and compensated Poisson measure is $\tilde{N}(d s, d w)=N(ds,dw)-\nu(d w)ds$. $\eta_k$ represents the $k$-th column vector of $\eta$.

For a function $h(x)$ on $\mathbb{R}^d$ that satisfies local Lipschitz continuity and linear growth, its Lipschitz norm $||h||_K$ is defined as
\begin{equation}
||h||_K := \sup_{x \in \mathbb{R}^d} \frac{|h(x)|}{1+|x|} + \sup_{\substack{x, y \in \mathbb{R}_K^d; x \neq y}} \frac{|h(x) - h(y)|}{|x - y|}.
\end{equation}

\begin{assumption}\label{Assumption 1} The following conditions hold.
\item[$(\romannumeral 1)$] The spaces $\mathcal{Y}$ and $\mathcal{Z}$ are compact metric spaces, and for each $i \in \mathcal{S}$, $\sigma$, $r$, $b$, and $\eta$ are continuous functions of $t$, $x$, $y$, $z$ (for $\eta$, $w$ are also included);
\item[$(\romannumeral 2)$] $\exists C_K > 0$ such that $\left(||b||_K+||\sigma||_K\right)(t,\cdot,i; y, z)\leq C_K$ and $||\eta(t, \cdot, i; y, z; w)||_K \leq C_K \min(|w|, 1)$ for all $t$, $i$, $y$, $z$ and $w$;
\item[$(\romannumeral 3)$] The L\'{e}vy measure $\nu$ is a positive measure that satisfies $\int_{E} \min\left(|w|^2, 1\right) \nu(dw) < \infty$.
\item[$(\romannumeral 4)$] $\exists C > 0,p\geq 1$ s.t. $\forall t,x,i,y,z$, $|r(t,x,i; y, z)|+|g(x,i)|\leq C\left(1+|x|^p\right)$;
\item[$(\romannumeral 5)$] The Markov chain $\theta$ is independent of the Brownian motion $W$ and the Poisson random measure $N$.
\end{assumption}

Under Assumption \ref{Assumption 1}, the conditions guarantee that the SDE \eqref{CLSDE} and the payoff function \eqref{CLpayoff function} are well-defined.

In the ZSSDG, player $I$ and player $II$ maximize and minimize the payoff function $J^{cl}$ by controlling $y$ and $z$, respectively. At any time $s \in (t, T)$, both players observe the historical states $(X, \theta, y, z)(r)$ for $r < s$. As a result, the player who acts first at time $s$ is clearly at a disadvantage. Following \cite{fleming1989existence}, the problem is formalized into two approximate games: the upper game and the lower game. In the upper game, player $I$ can know $z_s$ before choosing $y_s$; in the lower game, player $II$ can know $y_s$ before choosing $z_s$. \cite{biswas2012zero} has established the DPP for this game under a single regime (i.e., $\mathcal{S} = \{1\}$), and provided a verification theorem for viscosity solutions.

\subsection{Challgens for building ERRL regime-switching jump-diffusion SDEs}\label{sub22}
We perceive that there are three main aspects of the model proposed by \cite{biswas2012zero} that can be further explored.

First, it assumes that all parameters are known, whereas in practice, both model parameters and structures are often uncertain. The traditional \lq\lq separation principle" decouples parameter estimation from decision optimization: historical data are first used to estimate the parameters, and the estimated values are then treated as true inputs for optimization (\cite{gao2025bayesian}). While this approach may be valid asymptotically, in many cases the sample size required for reliable estimation far exceeds the available data (\cite{merton1980estimating}).

Second, deterministic optimal strategies derived under known parameters are critically sensitive to estimation errors, which tend to be amplified through the optimization procedure, thereby undermining the practical applicability of theoretical solutions. In contrast, modern RL operates under a fundamentally different paradigm. While RL also presumes that data originate from an underlying structural model-such as a Markov process---it does not rely on prior knowledge or explicit parameter estimation. Instead, optimal policies are learned directly through interaction with data. This rationale naturally motivates the incorporation of randomized exploration, which broadens the search space, enhances flexibility, and supports learning through trial and error. Pioneering work by \cite{wang2020reinforcement} and \cite{wang2020continuous} introduced this approach to stochastic control in diffusion models, while \cite{sun2023reinforcement} later formulated an exploratory ZSSDG framework for linear-quadratic regulators in diffusion processes.

Third, although regime-switching effects have been increasingly considered in various control and game-theoretic settings \cite{song2008numerical,Kushner2004numerical,gruen2012probabilistic,Huang2025convergence,feng2021optimal,Nguyen2021general,chen2022stochastic,gutierrez2024markovian}, state-dependent structural shifts are ubiquitous in real-world economic, financial, and engineering dynamical systems \cite{jin2013optimal,lv2024solving}. Regime-switching causes discrete transitions in the drift, diffusion, and jump intensity of the system according to external environments or internal states, essentially introducing extra stochastic structure and mode-dependent characteristics. Within the jump-diffusion and exploratory RL framework, regime-switching further complicates the model: on the one hand, the switching mechanism itself is stochastic and time-varying, making it incompatible with conventional parameter estimation and policy optimization methods; on the other hand, the three types of uncertainty, namely regime-switching, jump behavior, and randomized exploration, are mutually coupled, which breaks the logic of moment averaging and policy exploration construction in classical diffusion models and directly invalidates the original exploratory dynamics modeling approach.

Given the importance of regime-switching jump-diffusion models in SDGs, developing an ERRL framework for ZSSDGs with regime-switching and exploratory dynamics is of great theoretical and practical significance. This subsection outlines the core challenges in constructing such a model and presents the methodology adopted to address them.

For background, in the pure diffusion setting with regime-switching, the state dynamics are governed by a continuous-time Markov chain $\theta_t \in \mathcal{S} = \{1, 2, \dots, m\}$ and are given by
\[
dX_t = b(t, X_t, \theta_t, a_t) dt + \sigma(t, X_t, \theta_t, a_t) dW_t.
\]
When employing a stochastic policy $\boldsymbol{\pi}(\cdot|t,x,i)$ conditional on the current regime $i \in \mathcal{S}$, applying the Law of Large Numbers (LLN) yields the averaged drift and diffusion coefficients
\[
\begin{aligned}
   \frac{1}{N} \sum_{k=1}^N \Delta X_t^k  &\approx \int_{\mathcal{A}} b(t, x, i, a) \,\boldsymbol{\pi}(a|t,x,i) da \, \Delta t \\
   &= \mathbb{E}_{a \sim \boldsymbol{\pi}}[b(t,x,i,a)]\Delta t = \tilde{b}(t, x, i, \boldsymbol{\pi})\Delta t, \\
   \frac{1}{N}\sum_{k=1}^N (\Delta X_t^{k})^2 &\approx \int_{\mathcal{A}} \sigma^2(t, x, i, a)\, \boldsymbol{\pi}(a|t,x,i) da\, \Delta t \\
   &= \mathbb{E}_{a \sim \boldsymbol{\pi}}[\sigma^2(t,x,i,a)]\Delta t = \tilde{\sigma}^2(t, x, i, \boldsymbol{\pi})\Delta t.
\end{aligned}
\]
Consequently, the exploratory state process is
\[
d\tilde{X}_t = \tilde{b}(t, \tilde{X}_t, \theta_t, \boldsymbol{\pi}) dt + \tilde{\sigma}(t, \tilde{X}_t, \theta_t, \boldsymbol{\pi}) dW_t.
\]
The key observation is that a diffusion process is fully characterized by its first and second moments (drift and volatility); thus, averaging these moments suffices to construct the exploratory dynamics.

However, this straightforward moment-averaging approach fails for jump-diffusion processes. The state dynamics in this generalized regime-switching setting are
\begin{equation}
dX_t = b(t, X_t, \theta_t, a_t) dt + \sigma(t, X_t, \theta_t, a_t) dW_t+ \int_{\mathbb{R}^+} \eta(t, X_{t-}, \theta_{t-}, a_t; w) \tilde{N}(dt, dw),
\end{equation}
where $\tilde{N}(dt, dw) = N(dt, dw) - \nu(dw)dt$ is the compensated Poisson random measure and $\nu(dw)$ is the L\'evy measure. A jump process is characterized by three components: the drift, the diffusion, and the L\'evy measure describing the jump distribution.

The difficulty arises if one attempts to naively average the jump coefficient as $\tilde{\eta}(t, x, i, \boldsymbol{\pi}; w) = \mathbb{E}_{a \sim \boldsymbol{\pi}}[\eta(t,x,i,a;w)]$ and then use the term $\int_{\mathbb{R}^+} \tilde{\eta}(t, x, i, \boldsymbol{\pi}; w) \tilde{N}(dt, dw)$. This leads to two fundamental problems. First, it risks losing the discrete "event nature" of jumps; a non-zero $\tilde{\eta}$ might be misinterpreted as a continuous drift rather than discrete events. Second, and more critically, it alters the true distribution of jump sizes. For instance, consider a policy that randomizes between two actions: one causing a jump of magnitude $M$ with probability $p$, and another causing a jump of magnitude $m$ with probability $1-p$. Simple averaging yields an effective jump size of $pM + (1-p)m$, which represents a dynamic fundamentally different from the original mixture of large ($M$) and small ($m$) jumps.

Mathematically, the core difficulty is nonlinearity. Consider the jump and regime-switching components of the infinitesimal generator for a fixed action $a$ and current regime $i$, with a transition rate matrix $Q = (q_{ij})_{m \times m}$:
\[
\begin{aligned}
\mathcal{L}^a f(t,x,i) \supset&
 \int_{\mathbb{R}^+} \Big[ f(t, x+\eta(t,x,i,a;w), i) - f(t,x,i)\\
&- \eta(t,x,i,a;w) \cdot  f_x(t,x,i) \Big] \nu(dw)+\sum_{j=1}^m q_{ij} f(t,x,j) .
\end{aligned}
\]
While the regime transition sum is linear and averages trivially, the jump term does not. Averaging this expression over actions sampled from $\boldsymbol{\pi}$ yields a double integral
\begin{equation}\label{tab1}
\begin{aligned}
\int_{\mathcal{A}}& \int_{\mathbb{R}^+} \Big[ f(t, x+\eta(t,x,i,a;w), i) - f(t,x,i) \\
&- \eta(t,x,i,a;w) \cdot  f_x(t,x,i) \Big] \nu(dw) \boldsymbol{\pi}(a|t,x,i) da.
\end{aligned}
\end{equation}
Because $f(t,x+\eta, i)$ is a nonlinear function of $\eta$, the expression in \eqref{tab1} cannot be simplified to a form with a single averaged $\tilde{\eta}$, i.e.,
\[
\begin{aligned}
\int_{\mathbb{R}^+}& \Big[ f(t, x+\tilde{\eta}(t,x,i,\boldsymbol{\pi};w), i)\\
&- f(t,x,i) - \tilde{\eta}(t,x,i,\boldsymbol{\pi};w) \cdot  f_x(t,x,i) \Big] \nu(dw),
\end{aligned}
\]
except in the trivial case where $f$ is linear in $x$.

To address this challenge, we adopt the grid-sampling approach from \cite{gao2024reinforcement} to avoid measurability issues associated with continuous-time action randomization. Actions are sampled only on a discrete time grid $\mathbb{S} = \{0=t_0 < t_1 < \dots < t_N=T\}$. On the interval $(t_{k-1}, t_k]$, the action is held constant as $a_k = G^{\boldsymbol{\pi}}(t_{k-1}, X_{t_{k-1}}, \theta_{t_{k-1}}, U_k)$, where $U_k \sim \text{Uniform}[0,1]$. Here, $G^{\boldsymbol{\pi}}:[0, T] \times \mathbb{R} \times \mathcal{S} \times [0,1] \rightarrow \mathcal{A}$ is a function that transforms a uniform random vector into an action distributed according to $\boldsymbol{\pi}(\cdot|t,x,i)$; formally, for any $(t,x,i)$, if $U \sim \text{Uniform}[0,1]$, then $a=G^{\boldsymbol{\pi}}(t, x, i, U) \sim \boldsymbol{\pi}(\cdot \mid t, x, i)$. The grid-sampled state process then follows
\[
\begin{aligned}
dX_s^{\boldsymbol{\pi},\mathbb{S}} =& b(s, X_{s}^{\boldsymbol{\pi},\mathbb{S}}, \theta_{s}, a_k) ds+ \sigma(s, X_{s}^{\boldsymbol{\pi},\mathbb{S}}, \theta_{s}, a_k) dW_s \\
&+\int_{\mathbb{R}^+} \eta(s, X_{s-}^{\boldsymbol{\pi},\mathbb{S}}, \theta_{s-}, a_k;w) \tilde{N}(ds, dw), \,s \in (t_{k-1}, t_k].
\end{aligned}
\]
The key to this solution is to reparameterize the action randomization using $G^{\boldsymbol{\pi}}$ and then embed it directly into an enriched Poisson random measure. Define a new Poisson random measure $N'(dt, dw, du)$ on $\mathbb{R}^+ \times \mathbb{R}^+ \times [0,1]$ with intensity measure $\nu(dw) du dt$, and denote its compensated measure by $\tilde{N}^{\prime}$. The exploratory state process is then formally defined as
\[
\begin{aligned}
d\tilde{X}_t^{\boldsymbol{\pi}} =&\tilde{b}(t, \tilde{X}_{t}^{\boldsymbol{\pi}}, \theta_{t}, \boldsymbol{\pi}) dt + \tilde{\sigma}(t, \tilde{X}_{t}^{\boldsymbol{\pi}}, \theta_{t}, \boldsymbol{\pi}) dW_t \\
&+\int_{\mathbb{R}^+ \times [0,1]} \eta(t, \tilde{X}_{t-}^{\boldsymbol{\pi}}, \theta_{t-}, G^{\boldsymbol{\pi}}(t,\tilde{X}_{t-}^{\boldsymbol{\pi}}, \theta_{t-}, u); w)\tilde{N}^{\prime}(dt, dw, du).
\end{aligned}
\]
The corresponding infinitesimal generator for this process at state $(t, x, i)$ is
\[
\begin{aligned}
\mathcal{L}^{\boldsymbol{\pi}} f&(t,x,i)= \tilde{b} \cdot f_x + \frac{1}{2}\tilde{\sigma}^2 \cdot f_{xx} + \sum_{j=1}^m q_{ij} f(t,x,j) \\
&+ \int_{\mathbb{R}^+ \times [0,1]} \Big[ f(t, x+\eta(t,x,i,G^{\boldsymbol{\pi}}(t,x,i,u);w), i)- f(t,x,i) - \eta \cdot f_x \Big] \nu(dw) du.
\end{aligned}
\]
Crucially, this generator is exactly the average of the classical generators across actions:
\[
\mathcal{L}^{\boldsymbol{\pi}} f(t,x,i) = \int_{\mathcal{A}} \mathcal{L}^a f(t,x,i) \boldsymbol{\pi}(a|t,x,i) da.
\]
Philosophically, this method does not attempt to simplify the inherent nonlinearity $f(t,x+\eta(a), i)$. Instead, it accepts this nonlinearity and embeds the policy randomization into the jump measure via the function $G^{\boldsymbol{\pi}}$. This construction yields a well-defined stochastic process whose generator automatically computes the correct average $\mathbb{E}_{a \sim \boldsymbol{\pi}}[f(t,x+\eta(a), i)]$, thereby preserving the essential structure needed for RL algorithms.

\subsection{Entropy regularized exploratory SDG framework}\label{sub23}
A key concept in RL is exploring the unknown environment through randomized actions. Randomization employs probability distributions (measures) to substitute deterministic actions for the purpose of exploration. Entropy is used to quantify the dispersion of these probability distributions, where low entropy indicates concentration and high entropy signifies greater randomness. While there are various methods to determine distributional controls, this paper focuses solely on feedback distributional controls, i.e., the current policy depends only on the current state $t, x$ and the regime state $i$.

Let $\boldsymbol{\pi}$ : $(t, x, i) \in[0, T] \times \mathbb{R}^d \times \mathcal{S} \rightarrow \boldsymbol{\pi}(\cdot \mid t, x, i) \in \mathcal{P}(\mathcal{Y})$ $(resp.,\boldsymbol{\gamma}$ : $(t, x, i) \in[0, T] \times \mathbb{R}^d \times \mathcal{S} \rightarrow \boldsymbol{\gamma}(\cdot \mid t, x, i) \in \mathcal{P}(\mathcal{Z}))$ be a given stochastic feedback control, where $\mathcal{P}(\mathcal{Y})$ $(resp.,\mathcal{P}(\mathcal{Z}))$ is the set of probability density functions defined on $\mathcal{Y}$ $(resp.,\mathcal{Z})$. Player $I$ samples an action $y$ from $\boldsymbol{\pi}(\cdot|t, x, i)$, which can be represented as $y=G_{\boldsymbol{\pi}}(t,x,i,u)$, where $G_{\boldsymbol{\pi}}$ is a measurable function and $u$ follows a uniform distribution on $[0, 1]^{n1}$. Player $II$ samples an action $z$ from $\boldsymbol{\gamma}(\cdot|t, x, i)$, which can be expressed as $z=M_{\boldsymbol{\gamma}}(t,x,i,v)$, where $M_{\boldsymbol{\gamma}}$ is a measurable function and $v$ follows a uniform distribution on $[0, 1]^{n2}$. Next, we define the random control processes $\boldsymbol{\pi}_s = \boldsymbol{\pi}(\cdot | s, X_s, \theta_s)$ and $\boldsymbol{\gamma}_s = \boldsymbol{\gamma}(\cdot | s, X_s, \theta_s)$, and will also use the more concise notations $G_{\boldsymbol{\pi}_{s}}(u)$ for $G_{\boldsymbol{\pi}}(s, X_s, \theta_s,u)$ (resp., $M_{\boldsymbol{\gamma}_{s}}(v)$ for $M_{\boldsymbol{\gamma}}(s, X_s, \theta_s,v)$). Based on the method in \cite{gao2024reinforcement}, we extend the control process to an exploratory control process, and the corresponding regime-switching exploratory SDE is
\begin{equation}\label{SDE}
\left\{\begin{array}{ll}
dX_s=&\tilde{b}(s, X_{s}, \theta_{s} ; \boldsymbol{\pi}_s, \boldsymbol{\gamma}_s)ds+\tilde{\sigma}(s, X_{s}, \theta_{s} ; \boldsymbol{\pi}_s, \boldsymbol{\gamma}_s) d W_s\\
&+{\displaystyle\int}_{E\times[0,1]^{n1}\times[0,1]^{n2}} \eta(s, X_{s-}, \theta_{s-} ; G_{\boldsymbol{\pi}_s}(u), M_{\boldsymbol{\gamma}_s}(v); w)\tilde{N}^{\prime}(ds, dw, du, dv), \\
X_t=&x, \quad \theta_t=i,
\end{array} \right.
\end{equation}
where
\begin{equation*}
\begin{aligned}
\tilde{b}(t, x, i;\boldsymbol{\pi}, \boldsymbol{\gamma})&:=\int_{\mathcal{Y}}\left[\int_{\mathcal{Z}} b(t, x, i; y, z) \boldsymbol{\gamma}\,dz\right] \boldsymbol{\pi}\,dy,\\
\tilde{\sigma}(t, x, i; \boldsymbol{\pi}, \boldsymbol{\gamma})&:=\left\{\int_{\mathcal{Y}}\left[\int_{\mathcal{Z}} \sigma^2(t, x, i; y, z) \boldsymbol{\gamma}\,dz\right] \boldsymbol{\pi}\,dy\right\}^{1 / 2}.
\end{aligned}
\end{equation*}
$\tilde{N}^{\prime}(ds,dw,du,dv) = N^{\prime}(ds,dw,du,dv) - ds\nu(dw)dudv$ is the extended $\ell$-dimensional compensated Poisson measure, and $u$ $(resp.,v)$ is a real-valued vector following the $n1$$(resp.,n2)$-dimensional uniform distribution $\mathcal{U}([0, 1]^{n1})$ $(resp.,\mathcal{U}([0, 1]^{n2})$.

The exploratory payoff function is
\begin{equation}\label{payoff function}
\begin{array}{ll}
J(t, x, i; {\boldsymbol{\pi}}, {\boldsymbol{\gamma}})
=&\E_{t,x,i}\bigg[{\displaystyle\int}_t^T \Big(\tilde{r}(s, X_{s}, \theta_s; \boldsymbol{\pi}_s, \boldsymbol{\gamma}_s)-\lambda_1{\displaystyle\int}_{\mathcal{Y}} \boldsymbol{\pi}_s\ln \boldsymbol{\pi}_s \,dy\\
&+\lambda_2{\displaystyle\int}_{\mathcal{Z}} \boldsymbol{\gamma}_s \ln \boldsymbol{\gamma}_s\, dz\Big)ds+g(X_{T}, \theta_T)\bigg],
\end{array}
\end{equation}
where
\[
\tilde{r}(t, x, i; \boldsymbol{\pi},\boldsymbol{\gamma}):=\int_{\mathcal{Y}}\left[\int_{\mathcal{Z}} r(t, x, i; y, z)\, \boldsymbol{\gamma}\,dz\right] \boldsymbol{\pi}\,dy,
\]
and $\lambda_1$, $\lambda_2>0$ are called the \textit{temperature parameters}.
\begin{remark}
In ERRL, entropy encourages exploration by increasing policy randomness. Player $I$ maximizes the payoff, so the positive entropy term $-\lambda_1 \int_{\mathcal{Y}} \boldsymbol{\pi} \ln \boldsymbol{\pi} dy$ acts as a reward. Player $II$ minimizes the payoff, making the negative entropy term $\lambda_2 \int_{\mathcal{Z}} \boldsymbol{\gamma} \ln \boldsymbol{\gamma} dz$ a reward for them. The parameters $\lambda_1$ and $\lambda_2$ control the \textit{temperature parameter}, influencing how quickly each agent updates its policy using new information and the entropy incentive.
\end{remark}
\begin{assumption}\label{Assumption 2}
The following conditions hold.
\item[$(\romannumeral 1)$]
For all $t\in[0,T]$, $x\in \mathbb{R}_K^d$, $i \in \mathcal{S}$, $y$, $y^{\prime}\in\mathcal{Y}$, $z$, $z^{\prime}\in\mathcal{Z}$, and $w\in\mathbb{R}^\ell$, there exists $C_K>0$, such that $\int_{\mathbb{R}^+}\left|\eta_k(t, x, i; y, z; w)-\eta_k(t, x, i; y^{\prime}, z^{\prime}; w)\right|\nu_k(d w) \leq C_K\left(\left|y-y^{\prime}\right|+\left|z-z^{\prime}\right|\right)$;
\item[$(\romannumeral 2)$]
For all $t\in[0,T]$, $x$, $x^{\prime}\in \mathbb{R}_K^d$ and $i \in \mathcal{S}$  there exists a positive constant $C_K$ such that
$\int_{[0,1]^{n1}} \left|G_{\boldsymbol{\pi}}(t,x,i,u) - G_{\boldsymbol{\pi}}(t,x^{\prime},i,u)\right| du
+\int_{[0,1]^{n2}} \left|M_{\boldsymbol{\gamma}}(t,x,i,v) - M_{\boldsymbol{\gamma}}(t,x^{\prime},i,v)\right| dv\leq C_K\left|x-x^{\prime}\right|.$
\end{assumption}
\begin{definition}
(Admissible control).
The controls $\boldsymbol{\pi} = \boldsymbol{\pi}(\cdot \mid t, x, i)$,  $\boldsymbol{\gamma} = \boldsymbol{\gamma}(\cdot \mid t, x, i)$ are called admissible, and denote the sets by $\mathcal{M}$ and $\mathcal{N}$, if
\item[$(\romannumeral 1)$]  $\forall(t, x, i) \in[0, T] \times \mathbb{R}^d\times\mathcal{S}$, $\boldsymbol{\pi}(\cdot\mid t,x,i) \in \mathcal{P}(\mathcal{Y})$, and $\boldsymbol{\gamma}(\cdot\mid t,x,i)\in \mathcal{P}(\mathcal{Z})$ a.s.;
\item[$(\romannumeral 2)$]  For all $t$, $x$, $x^{\prime}\in \mathbb{R}_K^d$ and $i \in \mathcal{S}$ there exists a positive constant $C_K$ such that, $\int_{\mathcal{Y}}\left|\boldsymbol{\pi}(y \mid t, x, i)-\boldsymbol{\pi}(y \mid t, x^{\prime}, i)\right| dy
+\int_{\mathcal{Z}}\left|\boldsymbol{\gamma}(z \mid t, x, i)-\boldsymbol{\gamma}(z \mid t, x^{\prime}, i)\right| dz\leq C_K\left|x-x^{\prime}\right|;$
\item[$(\romannumeral 3)$] For all $t$, $x$, $i \in \mathcal{S}$ and some $p\geq2$, there exists a positive constant $C$ such that $\int_{\mathcal{Y}}\left|\ln \boldsymbol{\pi}(y \mid t, x, i)\right| \boldsymbol{\pi}(y \mid t, x, i) \,dy
+\int_{\mathcal{Z}}\left|\ln \boldsymbol{\gamma}(z \mid t, x, i)\right| \boldsymbol{\gamma}(z \mid t, x, i)\,dz\leq C\left(1+\left|x\right|^p\right);$
\item[$(\romannumeral 4)$] For all $t$, $x$, $i \in \mathcal{S}$ and some $p\geq1$, there exists a positive constant $C$ such that $\int_{\mathcal{Y}}\left|y\right|^{p} \boldsymbol{\pi}(y \mid t, x, i)\,dy+\int_{\mathcal{Z}}\left|z\right|^{p} \boldsymbol{\gamma}(z \mid t, x, i)\, dz
\leq C\left(1+\left|x\right|^p\right).$
\end{definition}
\begin{proposition}\label{Proposition 1}
Under Assumption \ref{Assumption 1} and Assumption \ref{Assumption 2}, for any admissible controls $\boldsymbol{\pi} = \boldsymbol{\pi}(\cdot \mid t, x, i)\in\mathcal{M}$ and $\boldsymbol{\gamma} = \boldsymbol{\gamma}(\cdot \mid t, x, i)\in\mathcal{N}$, the exploratory SDE \eqref{SDE} has a unique strong solution.
\end{proposition}
The proof can be found in Appendix \ref{Appendix B}.
\begin{remark}
For the stochastic feedback controls
\[
\begin{aligned}
\boldsymbol{\pi}(\cdot \mid t,x,i)& \sim \mathcal{N}\left(A_1(t,x,i), B_1(t,i)B_1(t,i)^\top\right),\\
\boldsymbol{\gamma}(\cdot \mid t,x,i)& \sim \mathcal{N}\left(A_2(t,x,i), B_2(t,i)B_2(t,i)^\top\right),
\end{aligned}
\]
we have the following representations using standard uniform random variables $u$ and $v$:
\[\begin{aligned}
G_{\boldsymbol{\pi}}(t,x,i,u) &= A_1(t,x,i) + B_1(t,i)\Phi^{-1}(u),\\
M_{\boldsymbol{\gamma}}(t,x,i,v) &= A_2(t,x,i) + B_2(t,i)\Phi^{-1}(v).
\end{aligned}
\]
It can be verified that if $A_1$ and $A_2$ satisfy the local Lipschitz condition with respect to $x$, then the stochastic feedback controls $\boldsymbol{\pi}$ and $\boldsymbol{\gamma}$ satisfy Assumption 2. Furthermore, if $A_1$ and $A_2$ also satisfy the linear growth condition with respect to $x$, and for any $(t,i)$ the covariance matrices $B_1(t,i)B_1(t,i)^\top$ and $B_2(t,i)B_2(t,i)^\top$ are positive definite, then these two stochastic feedback controls are admissible.
\end{remark}
Similar to the classical case, we give the definitions of the admissible strategies and value function.
\begin{definition}
(Admissible strategy). An admissible strategy $\boldsymbol{\alpha}$ $(resp., \boldsymbol{\beta}) $ for player $I$ $(resp., II )$ is a mapping $\boldsymbol{\alpha}:\mathcal{N}\rightarrow\mathcal{M} $ $(resp., \boldsymbol{\beta}: \mathcal{M}\rightarrow \mathcal{N})$. Such that $\boldsymbol{\alpha}[\boldsymbol{\gamma}]\in\mathcal{M}$ $(resp.,\boldsymbol{\beta}[\boldsymbol{\pi}]\in\mathcal{N})$, and if $\boldsymbol{\gamma}\approx\tilde{\boldsymbol{\gamma}}$ $(resp.,\boldsymbol{\pi}\approx\tilde{\boldsymbol{\pi}})$ on $[t,s]$, then $\boldsymbol{\alpha}[\boldsymbol{\gamma}]\approx\boldsymbol{\alpha}[\tilde{\boldsymbol{\gamma}}]$ $(resp.,\boldsymbol{\beta}[\boldsymbol{\pi}]\approx\boldsymbol{\beta}[\tilde{\boldsymbol{\pi}}])$ on $[t,s]$ for every $s\in[t,T]$. The set of all admissible strategies for player $I$ (resp., $II$ ) is denoted by $\Gamma(t)$  $(resp., \Delta(t))$.
\end{definition}
\begin{definition}
(Value function).
The exploratory upper value of the SDG \eqref{SDE}-\eqref{payoff function}, with initial data $(t, x, i)$, is given by
\begin{equation}\label{upper value}
\phi(t, x, i):=\sup _{\boldsymbol{\alpha} \in \Gamma(t)} \inf _{\boldsymbol{\gamma} \in \mathcal{N}(t)} J(t, x, i ; \boldsymbol{\alpha}[\boldsymbol{\gamma}], \boldsymbol{\gamma}).
\end{equation}
The exploratory lower value of the game is defined as
\begin{equation}\label{lower value}
\psi(t, x, i):=\inf _{\boldsymbol{\beta} \in \Delta(t)} \sup _{\boldsymbol{\pi} \in \mathcal{M}(t)} J(t, x, i ; \boldsymbol{\pi}, \boldsymbol{\beta}[\boldsymbol{\pi}]).
\end{equation}
\end{definition}
If $\phi(t, x, i)=\psi(t, x, i)$, then the game has a value in the sense of \cite{elliott1972existence}, and we call this common value the value of the game, denoted by $V(t, x, i)$. This means that
\[
V(t, x, i)=\phi(t, x, i)=\psi(t, x, i).
\]
The exploratory upper value HJBI equation associated with upper value \eqref{upper value} is
\begin{equation}\label{upper value HJBI}
\begin{aligned}
\phi_t(t, x, i)+\sup_{\boldsymbol{\alpha} \in \Gamma(t)} \inf _{\boldsymbol{\gamma} \in \mathcal{N}(t)}\tilde{\mathcal{H}}(t, x, i, \phi_x, \phi_{xx}; \boldsymbol{\alpha}[\boldsymbol{\gamma}], \boldsymbol{\gamma}; \phi(t, \cdot,i))=0&,\\
\quad\forall(t,x,i)\in[0, T) \times \mathbb{R}^d \times \mathcal{S}&,\\
\qquad\quad\phi(T,x,i)=g(x,i),\quad (x,i)\in \mathbb{R}^d \times \mathcal{S},\qquad\qquad\quad&
\end{aligned}
\end{equation}
and the exploratory lower value HJBI equation associated with lower value \eqref{lower value} is
\begin{equation}\label{lower value HJBI}
\begin{aligned}
\psi_t(t, x, i)+\inf _{\boldsymbol{\beta} \in \Delta(t)}\sup _{\boldsymbol{\pi} \in \mathcal{M}(t)}\tilde{\mathcal{H}}(t, x, i, \psi_x, \psi_{xx};\boldsymbol{\pi}, \boldsymbol{\beta}[\boldsymbol{\pi}]; \psi(t, \cdot,i))=0&,\\
\quad\forall(t,x,i)\in[0, T) \times \mathbb{R}^d \times \mathcal{S}&,\\
\psi(T,x,i)=g(x,i),\quad (x,i)\in \mathbb{R}^d \times \mathcal{S},\qquad\qquad\quad&
\end{aligned}
\end{equation}
where
\[
\begin{aligned}
&\tilde{\mathcal{H}}(t, x, i, p, A;\boldsymbol{\pi}
,\boldsymbol{\gamma}; \varphi(t, \cdot,i))\\
:=&\int_{\mathcal{Y}}\left[\int_{\mathcal{Z}}\mathcal{H}(t, x, i, p, A;y,z;\varphi(t,\cdot,i))\,\boldsymbol{\gamma}\,dz\right] \boldsymbol{\pi}\,dy\\
&+\sum_{j=1}^{m} q_{ij} \varphi(t, x, j)-\lambda_1\int_{\mathcal{Y}}\boldsymbol{\pi} \ln \boldsymbol{\pi}\,dy
+\lambda_2\int_{\mathcal{Z}} \boldsymbol{\gamma} \ln \boldsymbol{\gamma}\,dz,\\
&\mathcal{H}(t, x, i, p, A;y,z;\varphi(t,\cdot,i))\\
:=&r(t, x, i;y,z)+b(t, x, i; y,z) \circ p +  \frac{1}{2}\sigma^2(t, x, i; y,z) \circ A\\
&+\sum_{k=1}^{\ell} \int_{\mathbb{R}^+}\Big[\varphi(t, x+\eta_k(t, x, i; y,z; w), i)-\varphi(t, x, i)-\eta_k(t, x, i;y,z; w) \circ \varphi_x\Big] \nu_k(d w).\\
\end{aligned}
\]
It is said that the \textit{Isaacs condition} is hold if
\begin{equation}\label{Isaacs condition}
\begin{aligned}
&\sup _{\boldsymbol{\alpha} \in \Gamma(t)} \inf _{\boldsymbol{\gamma} \in \mathcal{N}(t)}\tilde{\mathcal{H}}(t, x, i, \phi_x, \phi_{xx};\boldsymbol{\alpha}[\boldsymbol{\gamma}], \boldsymbol{\gamma}; \phi(t, \cdot,i))\\
=&\inf _{\boldsymbol{\beta} \in \Delta(t)}\sup _{\boldsymbol{\pi} \in \mathcal{M}(t)}\tilde{\mathcal{H}}(t, x, i, \psi_x, \psi_{xx};\boldsymbol{\pi}, \boldsymbol{\beta}[\boldsymbol{\pi}]; \psi(t, \cdot,i)).
\end{aligned}
\end{equation}
Under the \textit{Isaacs condition}, the exploratory HJBI equations are equivalent-they share a solution, denoted $\psi$.

The goal of the game is to establish optimal feedback strategies $(\boldsymbol{\pi}^*, \boldsymbol{\gamma}^*)$ for player $I$ and player $II$ that satisfy the equilibrium. Once $(\boldsymbol{\pi}^*,\boldsymbol{\gamma}^*)$ is obtained, the value function can be determined by the following PIDE,
\[
\begin{aligned}
\psi_t(t, x, i)+\tilde{\mathcal{H}}(t, x, i, \psi_x, \psi_{xx}; \boldsymbol{\pi}^*,\boldsymbol{\gamma}^*; \psi(t, \cdot,i))=0,&\\
\forall(t,x,i)\in[0, T) \times \mathbb{R}^d \times \mathcal{S},&\\
\psi(T,x,i)=g(x,i),\quad (x,i)\in \mathbb{R}^d \times \mathcal{S}.\qquad\quad&
\end{aligned}
\]

Theorem \ref{Theorem Optimal response strategies} below provides the optimal response strategies for player $I$ and player $II$ given any arbitrary strategy of the opponent.
\begin{theorem}\label{Theorem Optimal response strategies}
(Optimal response strategies).
If the HJBI equations admit solutions $\phi(\cdot,\cdot,i), \psi(\cdot,\cdot,i) \in C^{1,2}([0, T] \times \mathbb{R}^d) \cap C([0, T] \times \mathbb{R}^d)$ for each $i \in \mathcal{S}$, respectively, then the optimal response strategies $\boldsymbol{\alpha}^*[\boldsymbol{\gamma}]$ and $\boldsymbol{\beta}^*[\boldsymbol{\pi}]$ are given by
\begin{equation*}
\begin{aligned}
\boldsymbol{\alpha}^*[\boldsymbol{\gamma}]&=\frac{\displaystyle\exp \left(\frac{1}{\lambda_1}\displaystyle\int_{\mathcal{Z}}\mathcal{H}(t,x,i,\phi_x,\phi_{xx};y,z;\phi(t,\cdot,i))\boldsymbol{\gamma}dz\right)}{\displaystyle\int_\mathcal{Y} \exp \left(\frac{1}{\lambda_1}\displaystyle\int_{\mathcal{Z}}\mathcal{H}(t,x,i,\phi_x,\phi_{xx};y,z;\phi(t,\cdot,i))\boldsymbol{\gamma} dz\right) dy}, \\
\boldsymbol{\beta}^*[\boldsymbol{\pi}]&=\frac{\displaystyle\exp \left(\frac{-1}{\lambda_2}\displaystyle\int_{\mathcal{Y}}\mathcal{H}(t,x,i,\psi_x,\psi_{xx};y,z;\psi(t,\cdot,i))\boldsymbol{\pi}dy\right)}{\displaystyle\int_\mathcal{Z} \exp \left(\frac{-1}{\lambda_2}\displaystyle\int_{\mathcal{Y}}\mathcal{H}(t,x,i,\psi_x,\psi_{xx};y,z;\psi(t,\cdot,i))\boldsymbol{\pi}dy\right) dz}.
\end{aligned}
\end{equation*}
\end{theorem}
\begin{proof}
Starting from the lower value HJBI equation,  $\forall\boldsymbol{\gamma}\in \mathcal{N}(t)$, the maximizer $\boldsymbol{\alpha}^*[\boldsymbol{\gamma}]$ is
\begin{equation*}
\begin{aligned}
\boldsymbol{\alpha}^*[\boldsymbol{\gamma}]\in\arg\sup_{\boldsymbol{\pi} \in \mathcal{M}(t)}\left\{- \lambda_1 \int_{\mathcal{Y}} \boldsymbol{\pi} \ln \boldsymbol{\pi}\,dy\int_{\mathcal{Y}} \left[ \int_{\mathcal{Z}} \mathcal{H}(t, x, i, \psi_x, \psi_{xx}; y, z; \psi(t, \cdot,i)) \boldsymbol{\gamma}\,dz \right] \boldsymbol{\pi}\,dy \right\}
\end{aligned}
\end{equation*}
subject to the constraints that  $\boldsymbol{\pi}$ is a probability distribution, namely
$\int_{\mathcal{Y}} \boldsymbol{\pi}\,dy = 1.$
Introducing the Lagrange multiplier $\mu$  to this problem immediately yields
\begin{equation*}
\mathcal{L}(\boldsymbol{\pi},\mu)= \int_{\mathcal{Y}} \bigg\{ \left[\int_{\mathcal{Z}} \mathcal{H}(t, x, i, \psi_x, \psi_{xx}; y, z; \psi(t, \cdot,i)) \boldsymbol{\gamma}\,dz\right] \boldsymbol{\pi}\\
- \lambda_1 \boldsymbol{\pi}\ln \boldsymbol{\pi}\bigg\} dy
-\mu \left( \int_{\mathcal{Y}} \boldsymbol{\pi}\,dy - 1 \right).
\end{equation*}
The stationary condition is
\begin{equation*}
\left\{\begin{array}{ll}
\delta \mathcal{L}&={\displaystyle\int}_{\mathcal{Y}} \bigg\{\left[ {\displaystyle\int}_{\mathcal{Z}} \mathcal{H}(t, x, i, \psi_x, \psi_{xx}; y, z; \psi(t, \cdot,i)) \boldsymbol{\gamma}\,dz \right]- \lambda_1 (1 + \ln \boldsymbol{\pi}) - \mu \bigg\}\delta\boldsymbol{\pi} dy=0,\\
0&={\displaystyle\int}_{\mathcal{Y}} \boldsymbol{\pi}\,dy - 1.
\end{array} \right.
\end{equation*}
Consequently, the maximizer $\boldsymbol{\alpha}^*[\boldsymbol{\gamma}]$ can be obtained by solving the above equation. Similarly, starting from the upper value HJBI equation, the minimizer $\boldsymbol{\beta}^*[\boldsymbol{\pi}]$ can also be derived.
\end{proof}

\section{Actor-Critic RL algorithm design}\label{Section 4}
Deriving analytical solutions for the HJBI equations in a regime-switching jump-diffusion ZSSDG is generally intractable due to the non-linear PIDEs and the coupling induced by the Markov chain generator $Q = (q_{ij})_{m \times m}$. Motivated by \cite{zhou2021actor}, \cite{jia2022policy} ,\cite{hao2022entropy}, and \cite{wang2020continuous}, we propose a Continuous-Time Actor-Critic RL algorithm to approximate the value function and the optimal strategies.

The algorithm alternates between Policy Evaluation (PE) and Policy Improvement (PI). In the PE step, the critic estimates the value function of the current policies via the Episodic Online Continuous-Time Temporal Difference (CTD) method, which minimizes the martingale residual. In the PI step, the actors update their policies by minimizing the Kullback-Leibler (KL) divergence between their parameterized distributions and the optimal Gibbs response distributions characterized in Theorems \ref{Policy improvement for player $I$} and \ref{Policy improvement for player $II$}.
\begin{theorem}\label{Policy improvement for player $I$}
(Policy improvement for player $I$)
Let $\boldsymbol{\gamma}$ be fixed and $\boldsymbol{\pi}$ be an arbitrarily given admissible strategy. Suppose that the corresponding value function $V^{\boldsymbol{\pi}}(\cdot,\cdot,i) \in C^{1,2}([0,T) \times \mathbb{R}^d) \cap C^{0}([0,T] \times \mathbb{R}^d)$ for each $i \in \mathcal{S}$, for any $(t,x,i)\in[0,T) \times \mathbb{R}^d \times \mathcal{S}$, the strategy $\tilde{\boldsymbol{\pi}}$ defined by
\begin{equation}\label{Policy improvement pi}
\scalebox{1.15}{$
\tilde{\boldsymbol{\pi}}=\frac{\displaystyle\exp \left(\frac{1}{\lambda_1}\displaystyle\int_{\mathcal{Z}}\mathcal{H}(t,x,i,V^{\boldsymbol{\pi}}_x,V^{\boldsymbol{\pi}}_{xx};y,z;V^{\boldsymbol{\pi}}(t,\cdot,i))\boldsymbol{\gamma}\, dz\right)}{\displaystyle\int_\mathcal{Y} \exp \left(\frac{1}{\lambda_1}\displaystyle\int_{\mathcal{Z}}\mathcal{H}(t,x,i,V^{\boldsymbol{\pi}}_x,V^{\boldsymbol{\pi}}_{xx};y,z;V^{\boldsymbol{\pi}}(t,\cdot,i))\boldsymbol{\gamma} \,dz\right)dy} $}\end{equation}
is admissible. Then, the value function under specific strategies $\tilde{\boldsymbol{\pi}}$ and $\boldsymbol{\pi}$ satisfies
\begin{equation}V^{\tilde{\boldsymbol{\pi}}}(t,x,i)\geq V^{\boldsymbol{\pi}}(t,x,i),\,(t,x,i)\in[0,T]\times\mathbb{R}^d \times \mathcal{S}.
\end{equation}
\end{theorem}
\begin{proof}
Fix $\boldsymbol{\gamma}$ and $(t, x, i)\in [0,T]\times\mathbb{R}^d \times \mathcal{S}$. Let $X_h^{\boldsymbol{\pi}}$ be the state for $(\boldsymbol{\pi}, \boldsymbol{\gamma})$, $X_h^{\tilde{\boldsymbol{\pi}}}$ be the state for $(\tilde{\boldsymbol{\pi}}, \boldsymbol{\gamma})$, and define
\begin{equation*}
\begin{aligned}
\tilde{\mathcal{J}}(t, x, i;\boldsymbol{\pi}
,\boldsymbol{\gamma};\phi(t, \cdot,i))
:=& \int_{\mathcal{Y}}\Bigg[\int_{\mathcal{Z}}\sum_{k=1}^{\ell}\int_{\mathbb{R}^+}\Big[\phi(t, x+\eta_k(t, x, i;y, z; w), i)\\
&-\phi( t, x, i)-\eta_k(t, x, i; y, z; w) \circ\phi_x\Big] \nu_k(dw)\boldsymbol{\gamma}\,dz\Bigg] \boldsymbol{\pi}\,dy.
\end{aligned}
\end{equation*}
By applying It\^{o}-L\'{e}vy's formula with regime switching, we have
\begin{equation*}
\begin{aligned}
&V^{\boldsymbol{\pi}}(s,X_{s}^{\boldsymbol{\pi}}, \theta_s)-V^{\boldsymbol{\pi}}(t,x, i)\\
=&\int^s_tV_x^{\boldsymbol{\pi}}(h, X_{h-}^{\boldsymbol{\pi}}, \theta_{h-})\circ\tilde{\sigma}(h,X_{h-}^{\boldsymbol{\pi}}, \theta_{h-};\boldsymbol{\pi}_h,\boldsymbol{\gamma}_h)dW_h+\int_t^s dM^{\theta}_h \\
&+\int_t^s\tilde{\mathcal{J}}(h,X_{h-}^{\boldsymbol{\pi}}, \theta_{h-};\boldsymbol{\pi}_h,\boldsymbol{\gamma}_h;V^{\boldsymbol{\pi}}(h,\cdot, \cdot)) dh+\int^s_{t}\Big[\tilde{b}(h,X_{h-}^{\boldsymbol{\pi}}, \theta_{h-};\boldsymbol{\pi}_h,\boldsymbol{\gamma}_h) \circ \\
&\quad V_x^{\boldsymbol{\pi}}(h,X_{h-}^{\boldsymbol{\pi}}, \theta_{h-})+\frac{1}{2} \tilde{\sigma}^2(h,X_{h-}^{\boldsymbol{\pi}}, \theta_{h-};\boldsymbol{\pi}_h,\boldsymbol{\gamma}_h) \circ V_{xx}^{\boldsymbol{\pi}}(h,X_{h-}^{\boldsymbol{\pi}}, \theta_{h-})\\
&\quad+V_h^{\boldsymbol{\pi}}(h,X_{h-}^{\boldsymbol{\pi}}, \theta_{h-})+\sum_{j=1}^m q_{\theta_{h-}, j} V^{\boldsymbol{\pi}}(h, X_{h-}^{\boldsymbol{\pi}}, j)\Big]dh\\
&+\sum_{k=1}^\ell\int_t^s\int_{\mathbb{R^+}\times[0,1]^{n1}\times[0,1]^{n2}}\Big[V^{\boldsymbol{\pi}}(h,X_{h-}^{\boldsymbol{\pi}}+\eta_k(h,X_{h-}^{\boldsymbol{\pi}}, \theta_{h-}; G_{\boldsymbol{\pi}_h}(u), M_{\boldsymbol{\gamma}_{h}}(v); w), \theta_{h-})\\
&\qquad-V^{\boldsymbol{\pi}}(h,X_{h-}^{\boldsymbol{\pi}}, \theta_{h-})\Big]\tilde{N}_k^{\prime}(dh,dw,du,dv),
\end{aligned}
\end{equation*}
where $M^{\theta}_h$ is the zero-mean local martingale associated with the jumps of the continuous-time Markov chain $\theta$.
Define the stopping times
$\tau_n := \inf\{s \geq t : \int_t^s |V_x^{\boldsymbol{\pi}}(h, X_{h-}^{\boldsymbol{\pi}}, \theta_{h-})\circ \tilde{\sigma}(h,X_{h-}^{\boldsymbol{\pi}}, \theta_{h-};\boldsymbol{\pi}_h,\boldsymbol{\gamma}_h)|^2dh \geq n \}$, for $n \geq 1$. Then, by taking the conditional expectation, the martingale terms vanish, and we obtain
\begin{equation}\label{Ja}
\begin{aligned}	
&V^{\boldsymbol{\pi}}(s\wedge\tau_n,X_{s\wedge\tau_n}^{\boldsymbol{\pi}}, \theta_{s\wedge\tau_n})-V^{\boldsymbol{\pi}}(t,x,i)\\
=&\E_{t,x,i}\left[\int_t^{s\wedge\tau_n}\tilde{\mathcal{J}}(h,X_{h-}^{\boldsymbol{\pi}}, \theta_{h-};\boldsymbol{\pi}_h,\boldsymbol{\gamma}_h;V^{\boldsymbol{\pi}}(h,\cdot,\cdot)) dh\right.\\
&+\int^{s\wedge\tau_n}_{t}\Big(\tilde{b}(h,X_{h-}^{\boldsymbol{\pi}}, \theta_{h-};\boldsymbol{\pi}_h,\boldsymbol{\gamma}_h) \circ V_x^{\boldsymbol{\pi}}(h, X_{h-}^{\boldsymbol{\pi}}, \theta_{h-})\\
&\quad+\frac{1}{2} \tilde{\sigma}^2(h,X_{h-}^{\boldsymbol{\pi}}, \theta_{h-};\boldsymbol{\pi}_h,\boldsymbol{\gamma}_h) \circ V_{xx}^{\boldsymbol{\pi}}(h, X_{h-}^{\boldsymbol{\pi}}, \theta_{h-})\\
&\quad\left.+V_t^{\boldsymbol{\pi}}(h, X_{h-}^{\boldsymbol{\pi}}, \theta_{h-})+\sum_{j=1}^m q_{\theta_{h-}, j} V^{\boldsymbol{\pi}}(h, X_{h-}^{\boldsymbol{\pi}}, j)\Big)dh\right].
\end{aligned}
\end{equation}
On the other hand $\forall(t,x,i)\in[0, T) \times \mathbb{R}^d \times \mathcal{S}$, we have
\begin{equation}\label{Ha}
\begin{aligned}
-V^{\boldsymbol{\pi}}_t(t,x,i)=&\tilde{\mathcal{H}}(t, x, i, V^{\boldsymbol{\pi}}_x, V^{\boldsymbol{\pi}}_{xx};\boldsymbol{\pi}, \boldsymbol{\gamma}; V^{\boldsymbol{\pi}}(t, \cdot,i))\\
\leq&\tilde{\mathcal{H}}(t, x, i, V^{\tilde{\boldsymbol{\pi}}}_x, V^{\tilde{\boldsymbol{\pi}}}_{xx};\tilde{\boldsymbol{\pi}}, \boldsymbol{\gamma}; V^{\tilde{\boldsymbol{\pi}}}(t, \cdot,i)).
\end{aligned}
\end{equation}
Substituting \eqref{Ha} into \eqref{Ja} yields
$$\begin{aligned}
V^{\boldsymbol{\pi}}(t,x,i)
\leq& \E_{t,x,i}\bigg[\int_t^{s\wedge\tau_n} \Big(\tilde{r}(h,X_{h-}^{\tilde{\boldsymbol{\pi}}}, \theta_{h-};\tilde{\boldsymbol{\pi}}_h,\boldsymbol{\gamma}_h)-\int_{\mathcal{Y}} \tilde{\boldsymbol{\pi}}_h \ln \tilde{\boldsymbol{\pi}}_h\, dy\\
&+\int_{\mathcal{Z}} \boldsymbol{\gamma}_h \ln\boldsymbol{\gamma}_h \,dz\Big)d h+V^{\boldsymbol{\pi}}(s\wedge\tau_n,X^{\boldsymbol{\pi}}_{s\wedge\tau_n}, \theta_{s\wedge\tau_n})\bigg].
\end{aligned}$$
Set $s = T$ and let $n \to \infty$. Then, combining with the condition $V^{\boldsymbol{\pi}}(T, X_T^{\boldsymbol{\pi}}, \theta_T)= g(X_T^{\boldsymbol{\pi}}, \theta_T)$, we derive that $\forall(t,x,i) \in [0,T] \times \mathbb{R}^d \times \mathcal{S}$,
$$\begin{aligned}
V^{\boldsymbol{\pi}}(t,x,i)
\leq& \E_{t,x,i}\bigg[\int_t^{T} \Big(\tilde{r}(h,X_{h-}^{\tilde{\boldsymbol{\pi}}}, \theta_{h-};\tilde{\boldsymbol{\pi}}_h,\boldsymbol{\gamma}_h)-\int_{\mathcal{Y}} \tilde{\boldsymbol{\pi}}_h \ln \tilde{\boldsymbol{\pi}}_h\, dy\\
&+\int_{\mathcal{Z}} \boldsymbol{\gamma}_h \ln\boldsymbol{\gamma}_h \,dz\Big)d h+ g(X_T^{\tilde{\boldsymbol{\pi}}}, \theta_T)\bigg]=V^{\tilde{\boldsymbol{\pi}}}(t,x,i).
\end{aligned}$$
\end{proof}
\begin{theorem}\label{Policy improvement for player $II$}
(Policy improvement for player $II$).
Let $\boldsymbol{\pi}$ be fixed and $\boldsymbol{\gamma}$ be an arbitrarily given admissible strategy. Suppose that the corresponding value function $V^{\boldsymbol{\gamma}}(\cdot,\cdot,i) \in C^{1,2}([0,T) \times \mathbb{R}^d) \cap C^{0}([0,T] \times \mathbb{R}^d)$ for each $i \in \mathcal{S}$, for any $(t,x,i)\in[0,T) \times \mathbb{R}^d \times \mathcal{S}$, the strategy $\tilde{\boldsymbol{\gamma}}$ defined by
\begin{equation}\label{Policy improvement gamma}
\scalebox{1.15}{$\tilde{\boldsymbol{\gamma}}=\frac{\displaystyle\exp \left(\frac{-1}{\lambda_2}\displaystyle\int_{\mathcal{Y}}\mathcal{H}(t,x,i,V^{\boldsymbol{\gamma}}_x,V^{\boldsymbol{\gamma}}_{xx};y,z;V^{\boldsymbol{\gamma}}(t,\cdot,\cdot))	 {\boldsymbol{\pi}}\,dy\right)}{\displaystyle\int_\mathcal{Z} \exp \left(\frac{-1}{\lambda_2}\displaystyle\int_{\mathcal{Y}}\mathcal{H}(t,x,i,V^{\boldsymbol{\gamma}}_x,V^{\boldsymbol{\gamma}}_{xx};y,z;V^{\boldsymbol{\gamma}}(t,\cdot,\cdot)){\boldsymbol{\pi}} \,dy\right) dz}
$}\end{equation}
is admissible. Then, the value function under specific strategies $\tilde{\boldsymbol{\gamma}}$ and $\boldsymbol{\gamma}$ satisfies
\begin{equation}
V^{\tilde{\boldsymbol{\gamma}}}(t,x,i)\leq V^{\boldsymbol{\gamma}}(t,x,i),\,(t,x,i)\in[0,T]\times\mathbb{R}^d \times \mathcal{S}.
\end{equation}
\end{theorem}
\begin{proof}
The proof is similar to that of Theorem \ref{Policy improvement for player $I$} and is omitted here.
\end{proof}
\subsection{Parameterization of value function and policies}
In general continuous state-action spaces, \eqref{Policy improvement pi} and \eqref{Policy improvement gamma} involve numerous integrals that are computationally prohibitive. Thus, we introduce parameterized function approximators (e.g., neural networks or polynomial basis expansions).

Let $\psi$ denote the parameters of the Critic network, $V^{\psi}(t, x, i)$, which approximates the value function for each regime $i \in \mathcal{S}$.
Let $\phi_1$ and $\phi_2$ denote the parameters of the Actor networks for player $I$ and player $II$, representing the stochastic policies $\boldsymbol{\pi}^{\phi_1}(y | t, x, i)$ and $\boldsymbol{\gamma}^{\phi_2}(z | t, x, i)$, respectively.

\subsection{Discrete-time interaction and regime-switching}
Learning is achieved through episodic interactions with the environment. We adopt a discrete-time interaction scheme over a fixed time horizon $T$. Choose a small step size $\Delta t > 0$ and define the grid $t_k = k\Delta t$ for $k = 0, 1, \dots, K-1$.

At time $t_k$, given the current state $X_{t_k}$ and regime $\theta_{t_k} = i$, the players sample actions $y_k \sim \boldsymbol{\pi}^{\phi_1}(\cdot | t_k, X_{t_k}, i)$ and $z_k \sim \boldsymbol{\gamma}^{\phi_2}(\cdot | t_k, X_{t_k}, i)$. The regime state $\theta_t$ transitions according to the generator matrix $Q = (q_{ij})_{m \times m}$. We simulate the regime transition over $\Delta t$ by drawing from the transition probabilities $P(\theta_{t_{k+1}}=j | \theta_{t_k}=i) \approx \delta_{ij} + q_{ij}\Delta t$, where $\delta_{ij}$ is the Kronecker delta satisfying $\delta_{ij}=1$ for $i=j$ and $\delta_{ij}=0$ otherwise.

The continuous state is then advanced using the Euler-Maruyama discretization for the jump-diffusion process:
\begin{equation}\label{Euler-Maruyama}
\begin{aligned}
X_{t_{k+1}} &= X_{t_k} + b(t_k, X_{t_k}, \theta_{t_k}; y_k, z_k)\Delta t+ \sigma(t_k, X_{t_k}, \theta_{t_k}; y_k, z_k)\Delta W_k \\
&\quad + \sum_{l=1}^{\ell} \eta_l(t_k, X_{t_k}, \theta_{t_k}; y_k, z_k; w_k) \Delta \tilde{N}_{l,k},
\end{aligned}
\end{equation}
where $\Delta W_k$ is the $k$-dimensional Brownian increment. For each $l \in \{1, \dots, \ell\}$, $\eta_l$ denotes the $l$-th column vector of the matrix function $\eta$, and $\Delta \tilde{N}_{l,k}$ is the compensated Poisson jump increment of the $l$-th component. The one-step realized reward is $r_k = r(t_k, X_{t_k}, \theta_{t_k}; y_k, z_k)$.

\subsection{Policy Evaluation (Critic Update)}
Under the exploratory policies $\boldsymbol{\pi}^{\phi_1}$ and $\boldsymbol{\gamma}^{\phi_2}$, the process
\[
M_t^{\psi}: = V^{\psi}(t, X_t, \theta_t) + \int_0^t \hat{\mathcal{H}}_{\lambda}^{\phi_1, \phi_2}(s, X_s, \theta_s) \, ds
\]
must be a martingale. Where the entropy-regularized running reward is:
\begin{equation}
\hat{\mathcal{H}}_{\lambda}^{\phi_1, \phi_2}(t_k, X_{t_k}, \theta_{t_k}) = r_k - \lambda_1 \ln \boldsymbol{\pi}^{\phi_1}(y_k | t_k, X_{t_k}, \theta_{t_k})+ \lambda_2 \ln \boldsymbol{\gamma}^{\phi_2}(z_k | t_k, X_{t_k}, \theta_{t_k}).
\end{equation}
Following the continuous-time TD approach, the one-step martingale residual (TD error) at time $t_k$ for non-terminal steps ($k < K-1$) is given by:
\begin{equation}\label{TD_error}
\Delta_k^\psi =V^\psi(t_{k+1}, X_{t_{k+1}}, \theta_{t_{k+1}}) - V^\psi(t_k, X_{t_k}, \theta_{t_k})+ \hat{\mathcal{H}}_{\lambda}^{\phi_1, \phi_2}(t_k, X_{t_k}, \theta_{t_k}) \Delta t.
\end{equation}
Crucially, to enforce the terminal boundary condition of the HJBI equation, the true terminal payoff $g(\cdot, \cdot)$ must be incorporated at the final step ($k = K-1$). The TD error transitions directly to the known terminal reward:
\begin{equation}\label{TD_error_terminal}
\Delta_{K-1}^\psi = g(X_T, \theta_T) - V^\psi(t_{K-1}, X_{t_{K-1}}, \theta_{t_{K-1}})+ \hat{\mathcal{H}}_{\lambda}^{\phi_1, \phi_2}(t_{K-1}, X_{t_{K-1}}, \theta_{t_{K-1}}) \Delta t.
\end{equation}
The Critic parameters $\psi$ are updated to minimize the mean squared TD error. Additionally, a boundary loss term can be added to explicitly regularize the network towards the terminal condition. The parameter update rule is:
\begin{align}
\psi& \leftarrow \psi + \eta_\psi \cdot\Delta_k^\psi \cdot\nabla_\psi V^\psi(t_k, X_{t_k}, \theta_{t_k}),\\
\psi& \leftarrow \psi - \eta_\psi \cdot\nabla_\psi \left(V^\psi(T, X_T, \theta_T) - g(X_T, \theta_T)\right)^2,
\end{align}
where $\eta_\psi > 0$ is the critic learning rate.

\subsection{Policy Improvement (Actor Updates)}
The policy improvement theorems \ref{Policy improvement for player $I$} and \ref{Policy improvement for player $II$} indicate that the optimal policies should align with the Gibbs distributions proportional to $\exp(\frac{1}{\lambda_1}\mathcal{H}^\psi)$ and $\exp(-\frac{1}{\lambda_2}\mathcal{H}^\psi)$. Minimizing the KL divergence between the parameterized actors and these target distributions is equivalent to maximizing (for player $I$) and minimizing (for player $II$) the expected Hamiltonian augmented by the entropy.

Let $\mathcal{H}_k^\psi$ be the empirical Hamiltonian evaluated at the sampled transition using the Critic network $V^\psi$:
\begin{equation*}
\mathcal{H}_k^\psi= r_k + b_k \circ V_x^\psi + \frac{1}{2}\sigma_k^2 \circ V_{xx}^\psi  + \sum_{l=1}^{\ell} \int\Big[ V^\psi(t_k, X_{t_k} + \eta_l, \theta_{t_k}) - V^\psi - \eta_l \circ V_x^\psi \Big] \nu_l(dw).
\end{equation*}
Minimizing the KL divergence between the parameterized actor distributions and the target Gibbs distributions is mathematically equivalent to optimizing the following objectives:
\begin{enumerate}
\item For player $I$: Minimizing $\mathbb{E} [  \lambda_1 \ln \boldsymbol{\pi}^{\phi_1}(y_k) -\mathcal{H}_k^\psi ]$;
\item For player $II$: Minimizing $\mathbb{E} [\lambda_2 \ln \boldsymbol{\gamma}^{\phi_2}(z_k) + \mathcal{H}_k^\psi]$.
\end{enumerate}
The Actor parameters $\phi_1$ and $\phi_2$ are updated via stochastic gradient descent (SGD):
\begin{align}
\phi_1 \leftarrow \phi_1 - \eta_{\phi_1} \cdot\nabla_{\phi_1} (\lambda_1 \ln \boldsymbol{\pi}^{\phi_1} - \mathcal{H}_k^\psi),\\
\phi_2 \leftarrow \phi_2 - \eta_{\phi_2}\cdot \nabla_{\phi_2} (\lambda_2 \ln \boldsymbol{\gamma}^{\phi_2} + \mathcal{H}_k^\psi),
\end{align}
where $\eta_{\phi_1},\eta_{\phi_2}>0$ are actor learning rates.

These updates push player $I$'s policy to assign higher probabilities to actions $y$ that yield higher Hamiltonian values, and player $II$ to actions $z$ that yield lower Hamiltonian values, scaled by their respective temperature parameters $\lambda_1$ and $\lambda_2$.

\subsection{Algorithm Pseudocode}
The complete Actor-Critic learning procedure for the regime-switching jump-diffusion ZSSDG is summarized in Algorithm \ref{Algorithm 1}.

\begin{algorithm}[h!]
\caption{Actor-Critic RL for Regime-Switching Jump-Diffusion ZSSDG}
\textbf{Input:} function $b$, $\sigma$, $\eta$; reward function $r$, $g$; Markov chain generator matrix $Q=(q_{ij})$, L\'{e}vy measures $\nu_l$; initial state $X_0$, initial regime $\theta_0 \in \mathcal{S}$, time horizon $T$, time step $\Delta t$, temperature parameters $\lambda_1, \lambda_2$, learning rates $\eta_\theta, \eta_{\phi_1}, \eta_{\phi_2}$, number of episodes $N_{ep}$. \\
\textbf{Initialization:} Initialize Critic network parameters $\psi$, and Actor network parameters $\phi_1, \phi_2$.\label{Algorithm 1}
\begin{algorithmic}[1]
\For{episode $n = 1$ to $N_{ep}$}
    \State Set $t_0 = 0, X_{t_0} = X_0, \theta_{t_0} = \theta_0$.
    \For{step $k = 0$ to $K-1$}
        \State \textbf{Action Sampling:}
        \State Sample $y_k \sim \boldsymbol{\pi}^{\phi_1}(\cdot | t_k, X_{t_k}, \theta_{t_k})$.
        \State Sample $z_k \sim \boldsymbol{\gamma}^{\phi_2}(\cdot | t_k, X_{t_k}, \theta_{t_k})$.
        \State \textbf{Environment Transition:}
        \State Determine $\theta_{t_{k+1}}$ using transition probabilities $P(j|i) \approx \delta_{ij} + q_{ij}\Delta t$.
         \State Advance state to $X_{t_{k+1}}$ using the Euler-Maruyama jump-diffusion update via Eq. \eqref{Euler-Maruyama}.
        \State \textbf{Policy Evaluation (Critic Update):}
        \State Compute $\mathcal{H}_k^\psi$ and  $\hat{\mathcal{H}}_{\lambda}^{\phi_1, \phi_2}$.
        \If{$k == K-1$}
            \State Apply terminal boundary regularization\\ $\psi \leftarrow \psi - \eta_\psi \cdot\nabla_\psi \left(V^\psi(T, X_T, \theta_T) - g(X_T, \theta_T)\right)^2$.
        \EndIf
        \State Compute TD error $\Delta_k^\psi$ via Eq. \eqref{TD_error}.
        \State Update Critic parameters: \\
        $\psi \leftarrow \psi + \eta_\psi\cdot \Delta_k^\psi \cdot\nabla_\psi V^\psi(t_k, X_{t_k}, \theta_{t_k})$.
        \State \textbf{Policy Improvement (Actor Update):}
        \State Update player $I$: \\
        $\phi_1 \leftarrow \phi_1 - \eta_{\phi_1} \cdot\nabla_{\phi_1} (\lambda_1 \ln \boldsymbol{\pi}^{\phi_1} - \mathcal{H}_k^\psi)$.
        \State Update player $II$:\\
        $\phi_2 \leftarrow \phi_2 - \eta_{\phi_2}\cdot \nabla_{\phi_2} (\lambda_2 \ln \boldsymbol{\gamma}^{\phi_2} + \mathcal{H}_k^\psi)$.
    \EndFor
\EndFor
\end{algorithmic}
\textbf{Output:} Learned value function $V^{\psi^*}$ and optimal policies $(\boldsymbol{\pi}^{\phi_1^*}, \boldsymbol{\gamma}^{\phi_2^*})$.
\end{algorithm}

\section{Applications to linear-quadratic ZSSDG}\label{Section 3}
The exploratory linear-quadratic problem was first proposed by \cite{wang2020reinforcement}, who proved that the optimal exploration strategy follows a Gaussian distribution in a pure diffusion setting. Subsequently, \cite{sun2023reinforcement} studied the two-player linear-quadratic ZSSDG within an exploratory framework. In this section, we consider the exploratory two-player linear-quadratic ZSSDG based on a regime-switching jump-diffusion model. We first review the classical regime-switching linear-quadratic ZSSDG by drawing on the model established by \cite{moon2021linear}. For the simplicity of exposition, we consider a one-dimensional continuous state case, and the conclusions can be extended to multi-dimensional cases.
\subsection{Model and solution}
Consider the following controlled regime-switching SDE
\[
\left\{\begin{array}{ll}
dX_s=(A(\theta_s)X_s+B_1(\theta_s)y_s+B_2(\theta_s)z_s)ds\\
\qquad\quad+(C(\theta_s)X_s+D_1(\theta_s)y_s+D_2(\theta_s)z_s)dW_s\\
\qquad\quad+{\displaystyle\int}_{\mathbb{R^+}}\Big(E(\theta_{s-}, w)X_{s-}+F_1(\theta_{s-}, w)y_s+F_2(\theta_{s-}, w)z_s\Big)\tilde{N}(d s, d w),\\
X_t=x, \quad \theta_t=i, \quad (x\in \mathbb{R}, i \in \mathcal{S}).
\end{array} \right.
\]
The payoff function is
\begin{equation}
J^{cl}(t, x, i ; y, z)
=\E_{t, x, i}\Bigg[\int_t^T\Big(Q_1(\theta_s)X_s^2 +R_1(\theta_s)y_s^2+R_2(\theta_s)z_s^2 \Big)ds+Q_2(\theta_T)X_T^2\Bigg].
\end{equation}
Player $I$ and player $II$ maximize and minimize the payoff function $J^{cl}(t, x, i ; y, z)$ by controlling $y$ and $z$, respectively.

Now we introduce the exploratory ZSSDG framework, and the exploratory SDE is
\[
\left\{\begin{array}{ll}
dX _s=\Big(A(\theta_s)X _s+B_1(\theta_s){\displaystyle\int}_{\mathcal{Y}}y\,\boldsymbol{\pi}_s dy+B_2(\theta_s){\displaystyle\int}_{\mathcal{Z}}z\,\boldsymbol{\gamma}_s dz\Big)ds\\
\qquad\quad+\Big(C(\theta_s)X _s+D_1(\theta_s){\displaystyle\int}_{\mathcal{Y}}y\,\boldsymbol{\pi}_s dy+D_2(\theta_s){\displaystyle\int}_{\mathcal{Z}}z\,\boldsymbol{\gamma}_s dz\Big)dW_s\\
\qquad\quad+{\displaystyle\int}_{\mathbb{R^+}\times[0,1]\times[0,1]}\Big(E(\theta_{s-}, w)X _{s-}
+F_1(\theta_{s-}, w)G_{\boldsymbol{\pi}_s}(u)\\
\qquad\qquad\qquad\qquad\qquad+F_2(\theta_{s-}, w)M_{\boldsymbol{\gamma}_s}(v)\Big)
\tilde{N}^\prime(d s, d w,du,dv),\\
X_t=x, \quad \theta_t=i, \quad(x\in \mathbb{R}, i \in \mathcal{S}),
\end{array} \right.
\]
where $\mathcal{Y}=[-M_y,M_y]\subseteq\mathbb{R}$ and $\mathcal{Z}=[-M_z,M_z]\subseteq\mathbb{R}$, with $M_y,M_z>0$ chosen sufficiently large and held fixed; the exploratory payoff function is given by
\[
\begin{aligned}
J\left(t, x, i; {\boldsymbol{\pi}}, {\boldsymbol{\gamma}}\right)
=&\E_{t, x, i}\Bigg\{\int_t^T\Bigg[Q_1(\theta_s)(X _s)^2 +R_1(\theta_s)\int_{\mathcal{Y}}y^2\,\boldsymbol{\pi}_s dy+R_2(\theta_s)\int_{\mathcal{Z}}z^2\,\boldsymbol{\gamma}_s dz\\
&-\lambda_1\int_{\mathcal{Y}} \boldsymbol{\pi}_s\ln \boldsymbol{\pi}_s \,dy+\lambda_2\int_{\mathcal{Z}} \boldsymbol{\gamma}_s\ln \boldsymbol{\gamma}_s \,dz\Bigg]ds+Q_2(\theta_T)(X _T)^2\Bigg\}.
\end{aligned}
\]
Player $I$ and player $II$ maximize and minimize the payoff function $J(t, x, i ;\boldsymbol{\pi}, \boldsymbol{\gamma})$ by controlling $\boldsymbol{\pi}$ and $\boldsymbol{\gamma}$, respectively.

We suppose that the exploratory value function takes the following form
\[
V(t, x, i) = \frac{1}{2} P(t, i) x^2 + S(t, i).
\]
For the convenience of the subsequent discussion, we introduce several notations:
\begin{equation*}\begin{aligned}
a(t, i):=&R_1(i) + \frac{1}{2} P(t, i) \left( D_1(i)^2 + \int_{\mathbb{R}^+} F_1(i, w)^2 \nu(dw) \right),\\
b(t, i):=&P(t, i) \left( B_1(i) + C(i) D_1(i) + \int_{\mathbb{R}^+} E(i, w) F_1(i, w) \nu(dw) \right), \\
c(t, i):=&P(t, i) \left( D_1(i) D_2(i) + \int_{\mathbb{R}^+} F_1(i, w) F_2(i, w) \nu(dw) \right),\\
d(t, i):=&R_2(i) + \frac{1}{2} P(t, i) \left( D_2(i)^2 + \int_{\mathbb{R}^+} F_2(i, w)^2 \nu(dw) \right),\\
e(t, i):=&P(t, i) \left( B_2(i) + C(i) D_2(i) + \int_{\mathbb{R}^+} E(i, w) F_2(i, w) \nu(dw) \right).
\end{aligned}
\end{equation*}
\begin{assumption}\label{Assumption 3} We assume that $a(t, i)<0$, $d(t, i)>0$ for all $t \in [0,T]$ and $i \in \mathcal{S}$.
\end{assumption}
\begin{theorem}\label{Theorem 3}
Under Assumption \ref{Assumption 3}, the Isaacs condition holds, and the optimal strategies are given by
\[\left\{\begin{array}{ll}
\boldsymbol{\pi}^*(y|t,x,i)=\mathcal{N}\left( {\displaystyle\frac{c(t, i) e(t, i) - 2b(t, i) d(t, i)}{4a(t, i)d(t, i) - c^2(t, i)}}x, {\displaystyle\frac{-\lambda_1}{2a(t, i)}} \right), \\	
\boldsymbol{\gamma}^*(z|t,x,i)=\mathcal{N}\left( {\displaystyle\frac{b(t, i) c(t, i) - 2a(t, i) e(t, i)}{4a(t, i)d(t, i) - c^2(t, i)}}x, {\displaystyle\frac{\lambda_2}{2d(t, i)}} \right),
\end{array}\right.
\]
where $P(t, i)$ and $S(t, i)$ are determined by the following coupled system of ordinary differential equations (ODE):
\[\left\{\begin{array}{ll}
{\displaystyle\frac{1}{2}}P_t(s, i)=- {\displaystyle\frac{P(s, i)}{2}}\left( C(i)^2 + {\displaystyle\int}_{\mathbb{R}^+} E(i, w)^2 \nu(dw) \right) \\
\qquad\qquad\quad+{\displaystyle\frac{a(s, i)e(s, i)^2 + b(s, i)^2d(s, i) - b(s, i)c(s, i)e(s, i)}{4a(s, i)d(s, i) - c^2(s, i)}}\\
\qquad\qquad\quad-Q_1(i) - P(s, i)A(i) - {\displaystyle\frac{1}{2}}\sum_{j=1}^m q_{ij} P(s, j), \\
S_t(s, i)={\displaystyle\frac{-\lambda_2 + \lambda_1}{2}} - {\displaystyle\frac{\lambda_1}{2}}\ln\left( -{\displaystyle\frac{\pi e\lambda_1}{a(s, i)}} \right) \\
\qquad\qquad\quad+ {\displaystyle\frac{\lambda_2}{2}}\ln\left( {\displaystyle\frac{\pi e\lambda_2}{d(s, i)}} \right) - \sum_{j=1}^m q_{ij} S(s, j), \\
P(T, i)=2Q_2(i),\quad S(T, i)=0. \\
\end{array}\right.
\]
\end{theorem}
\begin{proof}
Starting from the lower value HJBI equation,  $\forall\boldsymbol{\gamma}\in \mathcal{N}(t)$, it follows from Theorem \ref{Theorem Optimal response strategies} that the maximizer $\boldsymbol{\alpha}^*[\boldsymbol{\gamma}]$ is
$$\boldsymbol{\alpha}^*[\boldsymbol{\gamma}]=\frac{\displaystyle\exp \left(\frac{1}{\lambda_1}\displaystyle\int_{\mathcal{Z}}\mathcal{H}(t,x,i,\psi_x,\psi_{xx};y,z;\psi(t,\cdot, i))\boldsymbol{\gamma}\,dz\right)}{\displaystyle\int_\mathcal{Y} \exp\left( \frac{1}{\lambda_1}\displaystyle\int_{\mathcal{Z}}\mathcal{H}(t,x,i,\psi_x,\psi_{xx};y,z;\psi(t,\cdot, i))\boldsymbol{\gamma}\, dz\right) dy}.$$
where $\psi(t, x, i) = \frac{1}{2} P(t, i) x^2 + S(t, i)$, and
$$\begin{aligned}
&\mathcal{H}(t, x, i, \psi_x, \psi_{xx};y,z; \psi(t, \cdot, i))\\
=&Q_1(i) x^2+ R_1(i) y^2 + R_2(i) z^2+ \psi_x\left( A(i) x + B_1(i) y + B_2(i) z \right)\\
&+\frac{1}{2} \psi_{xx}\left( C(i) x + D_1(i) y + D_2(i) z \right)^2\\
&+\int_{\mathbb{R}^+}\Big[ \psi(t, x + E(i,w)x + F_1(i,w)y + F_2(i,w)z, i) \\
&\,\,\,- \psi(t, x, i)-\psi_x (E(i,w)x + F_1(i,w)y + F_2(i,w)z) \Big]\nu(d w).
\end{aligned}$$
Simplifying yields
$$\boldsymbol{\alpha}^*[\boldsymbol{\gamma}]=\mathcal{N}\left(-\frac{b(t, i)x+c(t, i)\int_{\mathbb{R}} z \boldsymbol{\gamma} dz}{2a(t, i)}, \frac{-\lambda_1}{2a(t, i)}\right).$$
According to Assumption \ref{Assumption 3}, $a(t, i) < 0$, which indicates that $\boldsymbol{\alpha}^*[\boldsymbol{\gamma}]$ exists.
Substitute $\boldsymbol{\alpha}^*[\boldsymbol{\gamma}]$ into formula
$$\begin{aligned}
\boldsymbol{\gamma}^*&\in\arg\inf_{\boldsymbol{\gamma}\in \mathcal{N}(t)}\left\{\lambda_2 \int_{\mathcal{Z}} \boldsymbol{\gamma} \ln \boldsymbol{\gamma}\,dz-\lambda_1 \int_{\mathcal{Y}}\boldsymbol{\alpha}^*[\boldsymbol{\gamma}]\ln \boldsymbol{\alpha}^*[\boldsymbol{\gamma}]dy\right.\\
&\quad\left.+\int_{\mathcal{Z}} \left[ \int_{\mathcal{Y}} \mathcal{H} \left(t, x, i, \psi_x, \psi_{xx}; y, z; \psi(t, \cdot, \cdot)\right)\boldsymbol{\alpha}^*[\boldsymbol{\gamma}]dy\right] \boldsymbol{\gamma}dz\right\}.
\end{aligned}$$
With similar discussion, the minimizer $\boldsymbol{\gamma}^*$ is
$$\boldsymbol{\gamma}^*=\mathcal{N}\left( \frac{b(t, i) c(t, i) - 2a(t, i) e(t, i)}{4a(t, i)d(t, i) - c^2(t, i)}x, \frac{\lambda_2}{2d(t, i)} \right).$$
According to Assumption \ref{Assumption 3}, $d(t, i)> 0$, which indicates that $\boldsymbol{\gamma}^*$ exists.

So for
\[\sup _{\boldsymbol{\alpha} \in \Gamma(t)} \inf _{\boldsymbol{\gamma} \in \mathcal{N}(t)}\tilde{\mathcal{H}}(t, x, i, \psi_x, \psi_{xx};\boldsymbol{\alpha}[\boldsymbol{\gamma}], \boldsymbol{\gamma}; \psi(t, \cdot, i)),\]
we have
$$\left\{\begin{array}{l}
\boldsymbol{\alpha}^*[\boldsymbol{\gamma}^*]=\mathcal{N}\left( {\displaystyle\frac{c(t, i) e(t, i) - 2b(t, i) d(t, i)}{4a(t, i)d(t, i) - c^2(t, i)}}x, {\displaystyle\frac{-\lambda_1}{2a(t, i)}} \right), \\	
\boldsymbol{\gamma}^*=\mathcal{N}\left( {\displaystyle\frac{b(t, i) c(t, i) - 2a(t, i) e(t, i)}{4a(t, i)d(t, i) - c^2(t, i)}}x, {\displaystyle\frac{\lambda_2}{2d(t, i)}} \right).
\end{array}\right.$$
Using a similar derivation for the inf-sup case, we can obtain $\boldsymbol{\alpha}^*[\boldsymbol{\gamma}^*]=\boldsymbol{\pi}^*$, $\boldsymbol{\gamma}^*=\boldsymbol{\beta}^*[\boldsymbol{\pi}^*]$.
Therefore, the \textit{Isaacs condition} \eqref{Isaacs condition} holds and the optimal strategies are obtained.

The current problem is to solve the following PIDE,
$$\begin{aligned}
\tilde{\mathcal{H}}(t, x, i, \psi_x, \psi_{xx};\boldsymbol{\pi}^*, \boldsymbol{\gamma}^*; \psi(t, \cdot, i))+\psi_t(t, x, i)=0,&\\
\forall(t,x,i)\in[0, T) \times \mathbb{R} \times \mathcal{S},&\\
\psi(T,x, i)=Q_2(i) x^2.\qquad\qquad\qquad&
\end{aligned}$$
Substituting $\psi(t, x, i) = \frac{1}{2} P(t, i) x^2 + S(t, i)$ and the optimal strategies into the PIDE, we have
\[\begin{aligned}
&\psi_t+\tilde{\mathcal{H}}(t, x, i, \psi_x, \psi_{xx};\boldsymbol{\pi}^*, \boldsymbol{\gamma}^*; \psi(t, \cdot, \cdot))\\
=&\frac{1}{2} P_t(t, i) x^2 + S_t(t, i)+\bigg[ Q_1(i) + \frac{1}{2}\sum_{j=1}^m q_{ij} P(t, j)\\
&+ P(t, i)A(i)+ \frac{P(t, i)}{2}\left( C(i)^2 + \int_{\mathbb{R}^+} E(i, w)^2 \nu(dw) \right)\\
&\left.- \frac{a(t, i)e(t, i)^2 + b(t, i)^2d(t, i) - b(t, i)c(t, i)e(t, i)}{4a(t, i)d(t, i) - c(t, i)^2} \right] x^2 \\
&+ \frac{\lambda_2 - \lambda_1}{2} + \frac{\lambda_1}{2}\ln\left( -\frac{\pi e\lambda_1}{a(t, i)} \right) - \frac{\lambda_2}{2}\ln\left( \frac{\pi e\lambda_2}{d(t, i)} \right)+ \sum_{j=1}^m q_{ij} S(t, j)=0,\\
\psi&(T, x, i) = \frac{1}{2} P(T, i)  X_T^2 + S(T, i)=Q_2(i) X_T^2.
\end{aligned}
\]
Matching the coefficients of $x^2$ and the constant terms, we conclude that the coupled ODE system holds.
\end{proof}
\subsection{Numerical example}
In general, it is difficult to obtain an analytical solution for the coupled ODE system derived in Theorem \ref{Theorem 3}. Owing to its semi-analytical nature, we numerically solve this system of ODEs using the Runge-Kutta-Fehlberg method (RK45). The system parameters are provided in Table \ref{Table 1}, with the error tolerance set to $10^{-6}$. Fig. \ref{fig:P(t)andS(t)} contains two subplots: the left shows the numerical solutions $P(t, i)$ for each regime $i \in \mathcal{S}$, and the right shows $S(t, i)$. Over $t\in[0,10]$, $P(t, i)$ is strictly decreasing, falling smoothly from its initial value to the terminal value $2Q_2(i)$ and exhibiting overall concavity. By contrast, $S(t, i)$ increases linearly over time for each regime.
Fig. \ref{fig:Optimal strategies} presents the optimal feedback strategies, showing that the strategy variance is independent of the continuous state $x$, although it depends on the regime.
\begin{table}[h!]
\caption{System parames}
\label{Table 1}
\centering
\begin{tabular}{@{}llll@{}}
\hline
\textbf{Params} & \textbf{Value} & \textbf{Params} & \textbf{Value} \\
\hline
$A(1)$   & $-0.5$ & $A(2)$   & $-0.6$ \\
$C(1)$   & $0.2$  & $C(2)$   & $0.25$ \\
$Q_1(1)$ & $1.0$  & $Q_1(2)$ & $1.2$  \\
$Q_2(1)$ & $0.5$  & $Q_2(2)$ & $0.4$  \\
$R_1(1)$ & $-3.0$ & $R_1(2)$ & $-3.2$ \\
$R_2(1)$ & $1.5$  & $R_2(2)$ & $1.8$  \\
$D_1(1)$ & $0.05$ & $D_1(2)$ & $0.06$ \\
$D_2(1)$ & $0.05$ & $D_2(2)$ & $0.06$ \\
$B_1(1)$ & $0.5$  & $B_1(2)$ & $0.55$ \\
$B_2(1)$ & $0.3$  & $B_2(2)$ & $0.35$ \\
$q_{11}$ & $-0.5$ & $q_{12}$ & $0.5$  \\
$q_{21}$ & $0.3$  & $q_{22}$ & $-0.3$ \\
$\lambda_1$ & $1$ & $\lambda_2$ & $2$ \\
$e$ & $2.718$ & $T$ & $10.0$ \\
$E(w, 1)$ & $\exp\{-0.5w\}$ & $E(w, 2)$ & $\exp\{-0.6w\}$ \\
$F_1(w, 1)$ & $w\exp\{-w\}$ & $F_1(w, 2)$ & $w\exp\{-1.1w\}$ \\
$F_2(w, 1)$ & $\exp\{-2w\}$ & $F_2(w, 2)$ & $\exp\{-2.2w\}$ \\
$\nu(w, 1)$ & $\exp\{-w\}$ & $\nu(w, 2)$ & $\exp\{-w\}$ \\
\hline
\end{tabular}
\end{table}
\begin{figure}[h!]
	\centering
	\includegraphics[width=4.5in]{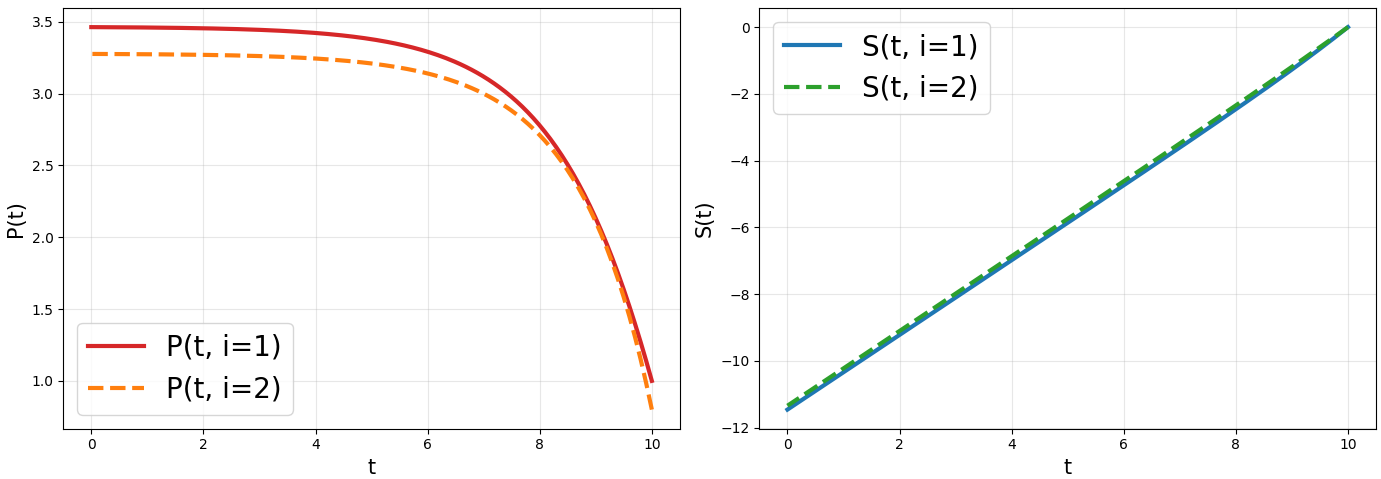}
	\caption{$P(t,i=1)$, $P(t,i=2)$, $S(t,i=1)$ and $S(t,i=2).$}
	\label{fig:P(t)andS(t)}
\end{figure}
\begin{figure}[h!]
	\centering
	\includegraphics[width=4.5in]{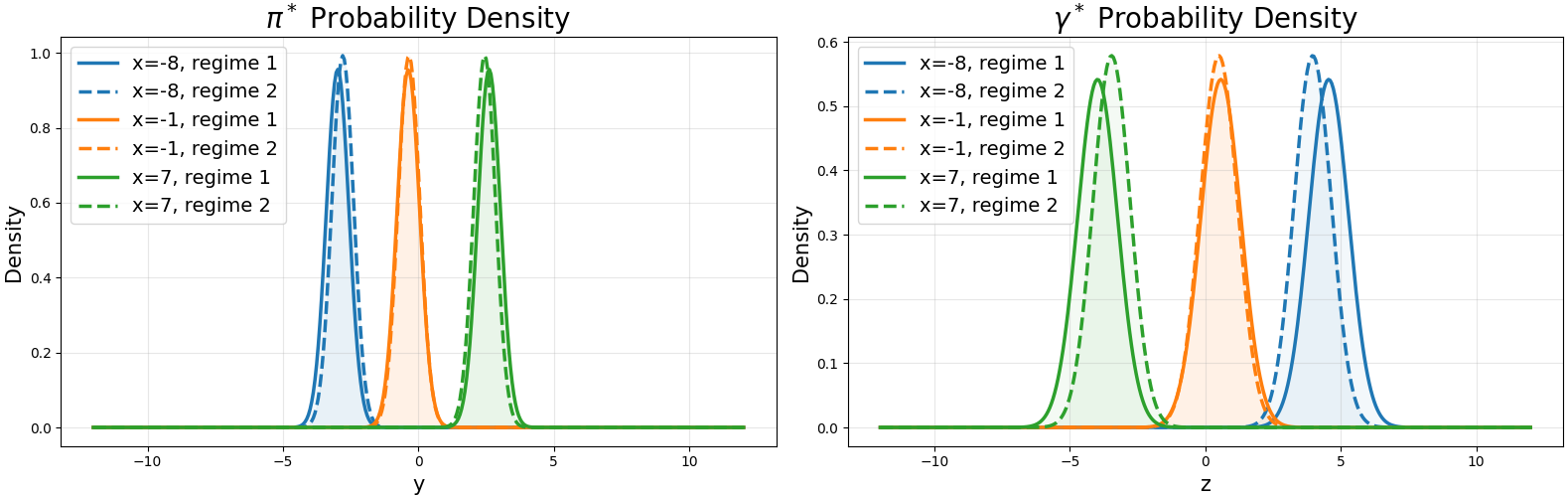}
	\caption{Optimal strategies $(\boldsymbol{\pi}^*, \boldsymbol{\gamma}^*)$ under regime $i\in \{1,2\}$.}
	\label{fig:Optimal strategies}
\end{figure}

We are interested in the relationship between the value function under exploratory strategies and that under classical strategies. To this end, we discuss the problem within the framework of the regime-switching linear-quadratic ZSSDG. We suppose that in the classical case, the value function takes the form $V^{cl}(t, x, i) = \frac{1}{2} P(t, i) x^2 + S(t, i)$. At this point, the following proposition holds.
\begin{proposition}\label{Proposition 2}
Under Assumption \ref{Assumption 3}, the classical optimal strategies in regime $i$ are given by
$$\left\{\begin{array}{ll}
y^*={\displaystyle\frac{c(t, i) e(t, i) - 2b(t, i) d(t, i)}{4a(t, i)d(t, i) - c^2(t, i)}}x, \\	
z^*={\displaystyle\frac{b(t, i) c(t, i) - 2a(t, i) e(t, i)}{4a(t, i)d(t, i) - c^2(t, i)}}x,
\end{array}\right.$$
where $P(t, i)$ and $S(t, i)$ are determined by the coupled ODE system:
$$\left\{\begin{array}{ll}
{\displaystyle\frac{1}{2}}P_t(s, i)=- {\displaystyle\frac{P(s, i)}{2}}\left( C(i)^2 + {\displaystyle\int}_{\mathbb{R}^+} E(i, w)^2 \nu(dw) \right)\\
\qquad\qquad\quad+{\displaystyle\frac{a(s, i)e(s, i)^2 + b(s, i)^2d(s, i) - b(s, i)c(s, i)e(s, i)}{4a(s, i)d(s, i) - c^2(s, i)}}\\
\qquad\qquad\quad-Q_1(i) - P(s, i)A(i) - {\displaystyle\frac{1}{2}}\sum_{j=1}^m q_{ij} P(s, j),\\
S(s, i)\equiv 0.
\end{array}\right.$$
\end{proposition}
\begin{proof}
The proof is standard and thus omitted here.
\end{proof}
By comparing the ODE system for $S(s, i)$ in Theorem \ref{Theorem 3} and Proposition \ref{Proposition 2}, it is evident that as the temperature parameters $\lambda_1, \lambda_2 \to 0^+$, the source terms in the exploratory ODE for $S(s, i)$ vanish. The system reduces to $S_t(s, i) + \sum_{j=1}^m q_{ij} S(s, j) = 0$ with $S(T, i) = 0$, which yields the unique solution $S(s, i) \equiv 0$. Therefore, we have
\begin{equation*}
\begin{aligned}
&\lim_{\lambda_1,\lambda_2\to0^+}\left(V(s, x, i)-V^{cl}(s, x, i)\right)\\
=&\lim_{\lambda_1,\lambda_2\to0^+}\bigg[-\frac{\lambda_2 - \lambda_1}{2} - \frac{\lambda_1}{2}\ln\left( -\frac{\pi e\lambda_1}{a(s)} \right)+ \frac{\lambda_2}{2}\ln\left( \frac{\pi e\lambda_2}{d(s)} \right)- \sum_{j=1}^m q_{ij} S(s, j)\bigg]=0.
\end{aligned}
\end{equation*}
The above results indicate that, at least in the setting of regime-switching linear-quadratic ZSSDG, the optimal exploratory strategy $\boldsymbol{\pi}^*$ (resp. $\boldsymbol{\gamma}^*$) converges to the classical strategy $y^*$ (resp. $z^*$) as the temperature parameter $\lambda_1$ (resp. $\lambda_2$) decreases, and the exploratory value function also converges to the classical value function as the temperature parameter diminishes.

It can be observed that the optimal strategies $(\boldsymbol{\pi}^*, \boldsymbol{\gamma}^*)$ for player $I$ and player $II$ both follow a Gaussian distribution, and their means are consistent with the optimal strategies under the classical scenario. Compared with the studies by \cite{wang2020reinforcement}, \cite{sun2023reinforcement}, and \cite{gao2024reinforcement}, the optimal strategy pair $(\boldsymbol{\pi}^*, \boldsymbol{\gamma}^*)$ obtained in this paper indicates that the jump term and regime-switching mechanics play a crucial role: they not only affect the variance of the optimal strategies but also influence their means through the coupled ODEs.

Specifically, the variance of the optimal strategy $\boldsymbol{\pi}^*$ for player $I$ in regime $i$ is only positively correlated with the weight parameters $(R_1(i), D_1(i), F_1(i, w))$ of its own control terms and the temperature parameter $\lambda_1$. This implies that the stronger player $I$'s ability to control the system within a specific regime, the higher the uncertainty of its strategy. In contrast, the variance of the optimal strategy $\boldsymbol{\gamma}^*$ for player $II$ is negatively correlated with the weight parameters $(R_2(i), D_2(i), F_2(i, w))$ of its own control terms, while being positively correlated with the temperature parameter $\lambda_2$. This suggests that the stronger player $II$'s control over the system, the more deterministic its strategy tends to be.

As $\lambda_1$ and $\lambda_2$ decrease, the optimal strategies approach the classical ones, and the exploratory value function converges to the classical value function. Conversely, as $\lambda_1$ and $\lambda_2$ increase, the variances of both players' optimal strategies rise, because a higher temperature parameter enlarges the search range over the decision space. In particular, when $E(i, w)=F_1(i, w)=F_2(i, w)=0$ for all $i$ (i.e., there is no jump term) and there is only a single regime ($\mathcal{S} = \{1\}$), our results reduce to those reported in \cite{sun2023reinforcement}.
\begin{remark}
The mean of the optimal feedback strategy pair $(\boldsymbol{\pi}^*, \boldsymbol{\gamma}^*)$ is linear in the state $x$. It can be verified that under this strategy pair, the exploratory regime-switching SDG is well-defined.
\end{remark}
\section{Applications to optimal investment game}\label{Section 5}
In this section, we apply the exploratory framework to a ZSSDG between an investor and the market under a regime-switching jump-diffusion model. For the classical formulation, we refer to \cite{savku2022stochastic}.
\subsection{Model}
Consider a financial market consisting of one risk-free asset (bond) and one risky asset. The states of the economy are modeled by a continuous-time Markov chain $\theta$ taking values in a finite state space $\mathcal{S} = \{1, 2, \ldots, m\}$, which represents different macroeconomic environments (e.g., bull and bear markets).

The price of the risk-free asset $S_0(s)$ evolves according to
$$
dS_0(s) = S_0(s) r(\theta_s) ds, \, S_0(0) = s_0 > 0,
$$
where $r(\theta_s)$ is the regime-dependent continuous compound risk-free interest rate.
The market acts as a player (player $II$) that can influence the expected return rate of the risky asset, denoted by the control variable $z_s \in \mathcal{Z} = \mathbb{R}^+$. The price of the risky asset $S_1(s)$ is modeled by the following regime-switching jump-diffusion process:
\begin{equation*}
\begin{aligned}
dS_1(s) =&z_sS_1(s) ds + \sigma(\theta_{s})S_1(s) dW_s\\
&+ \int_{\mathbb{R}_+} S_1(s-) \eta(\theta_{s-}, w) \tilde{N}(ds, dw), \, S_1(0) = s > 0,
\end{aligned}
\end{equation*}
where $\sigma(\theta_{s}) > 0$ is the volatility, and $\eta(\theta_{s-}, w)$ is the jump size of the risky asset.

The investor (player $I$) controls the proportion of wealth invested in the risky asset, denoted by $y_s \in \mathcal{Y} = [0,1]$. The remaining wealth proportion $1 - y_s$ is invested in the risk-free asset. The wealth process $X_s$ of the investor follows
$$
\left\{\begin{array}{ll}
dX_s = X_{s}\big( (1-y_s)r(\theta_s) + y_s z_s \big) ds + y_s \sigma(\theta_{s})X_{s} dW_s \\
\qquad\quad + \displaystyle\int_{\mathbb{R}_+} y_s X_{s-}\eta(\theta_{s-}, w) \tilde{N}(ds, dw), \\
X_t = x, \quad \theta_t = i.
\end{array} \right.
$$
The investor aims to maximize the expected utility of terminal wealth, while the market, acting as a fictitious opponent, seeks to minimize this expected utility. The payoff function of the classical zero-sum game is defined as
$$
J^{cl}(t, x, i; y, z) := \mathbb{E}_{t, x, i}\big[U(X_T, \theta_T)\big],
$$
where $U(\cdot, \cdot)$ is a regime-dependent utility function.

Now we introduce the exploratory ZSSDG framework. The investor and the market adopt stochastic feedback controls $\boldsymbol{\pi}_s$ and $\boldsymbol{\gamma}_s$, respectively. The exploratory wealth process is given by
\begin{equation}\label{Investment_SDE}
\left\{\begin{array}{ll}
dX_s = \left\{X_{s}\displaystyle \int_{\mathcal{Y}}\left[\int_{\mathcal{Z}} \big((1-y)r(\theta_s) + y z \big)\boldsymbol{\gamma}_sdz\right]\boldsymbol{\pi}_sdy \right\}ds \\
\qquad\quad +  \left( \displaystyle\int_{\mathcal{Y}} y^2 \sigma^2(\theta_s) \boldsymbol{\pi}_s dy \right)^{1/2}X_{s} dW_s \\
\qquad\quad + \displaystyle\int_{\mathbb{R}_+\times[0,1]} G_{\boldsymbol{\pi}_s}(u) X_{s-} \eta(\theta_{s-}, w) \tilde{N}^\prime(ds, dw, du), \\
X_t = x, \quad \theta_t = i.
\end{array} \right.
\end{equation}
The exploratory payoff function with entropy regularization is given by
\begin{equation}\label{Investment_pay}
\begin{aligned}
J\left(t, x, i; {\boldsymbol{\pi}}, {\boldsymbol{\gamma}}\right)
=&\E_{t, x, i}\bigg[\int_t^T\bigg(
-\lambda_1\int_{\mathcal{Y}} \boldsymbol{\pi}_s\ln \boldsymbol{\pi}_s \,dy\\&+\lambda_2\int_{\mathcal{Z}} \boldsymbol{\gamma}_s\ln \boldsymbol{\gamma}_s \,dz\bigg)ds+U(X _T, \theta_T)\bigg].
\end{aligned}
\end{equation}
\subsection{Numerical example}
Solving the exploratory SDG \eqref{Investment_SDE}-\eqref{Investment_pay} is highly challenging. This subsection will adopt the RL algorithm proposed in Section \ref{Section 4} to solve for the optimal strategies and the value function. We assume that the continuous-time Markov chain consists of two regimes ($m=2$), representing the bull market and the bear market, respectively.

For the optimal investment game, we adopt the standard power utility function given by
\begin{equation}
U(x) = \frac{x^\zeta}{\zeta}, \,\zeta < 1,\, \zeta \neq 0.
\end{equation}
Here, we assume the investor's intrinsic risk preference $\zeta$ remains constant, meaning the terminal utility evaluation is independent of the terminal macroeconomic regime. However, due to the regime-dependent market parameters and the Markovian transitions, the investor's optimal exploratory policy and the resulting value function will still dynamically adapt to the current regime $i$.

Table \ref{Table 2} presents the parameter settings. Figure \ref{fig:action_distributions_over_states} illustrates the optimal exploratory policies, $\boldsymbol{\pi}^*(y)$ and $\boldsymbol{\gamma}^*(z)$, for the investor (player $I$) and the market (player $II$) across different time-state combinations ($t=0.5$; $x=0.5$ and $x=1.0$).
\begin{table}[h!]
\caption{System parames}
\label{Table 2}
\centering
\begin{tabular}{@{}llll@{}}
\hline
\textbf{Params} & \textbf{Value} & \textbf{Params} & \textbf{Value} \\
\hline
$r(1)$ & $0.7$ & $r(2)$ & $0.4$ \\
$\sigma(1)$ & $0.15$ & $\sigma(2)$ & $0.30$ \\
$q_{11}$ & $-0.5$ & $q_{12}$ & $0.5$ \\
$q_{21}$ & $1.0$ & $q_{22}$ & $-1.0$ \\
$\lambda_1$ & $0.1$ & $\lambda_2$ & $0.1$ \\
$\eta(1, w)$ & $-0.1$ & $\eta(2, w)$ & $-0.3$ \\
$T$ & $1.0$ & $\zeta$ & $0.5$ \\
\hline
\end{tabular}
\end{table}
\begin{figure}[h!]
\centering
\includegraphics[width=4.5in]{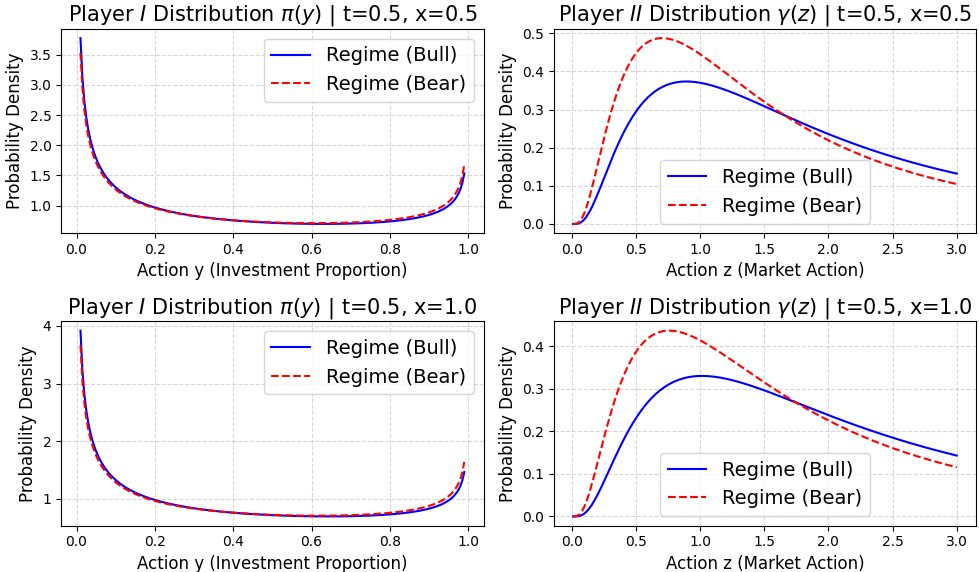}
\caption{Optimal strategy $\boldsymbol{\pi}^*$ and  $\boldsymbol{\gamma}^*$.}
\label{fig:action_distributions_over_states}
\end{figure}

For the investor, the probability density of the optimal strategy $\boldsymbol{\pi}^*(y)$ is highly concentrated near $y=0$. This indicates an extremely conservative investment stance, where the investor allocates almost no wealth to risky assets. Interestingly, due to the boundedness of the action space $[0, 1]$, the distribution exhibits a slight U-shaped characteristic; however, the density at the lower bound is exceptionally high, dominating the overall risk-averse strategy. Furthermore, comparing the top and bottom rows of the figure (for $x=0.5$ and $x=1.0$), it is evident that the investor's strategy is largely unaffected by the wealth level $x$. This is perfectly consistent with the wealth-independent property of the chosen power utility function.

In contrast, the market's optimal strategy $\boldsymbol{\gamma}^*(z)$ exhibits a unimodal, bell-shaped distribution, characterized by a high density in the central region and a gradual decay in the tails. The peak of this distribution consistently resides within the $z < 1$ interval. Furthermore, a clear regime-dependent behavior is observed: compared to the bull market regime (solid blue line), the mean of the market's optimal strategy $\boldsymbol{\gamma}^*(z)$ shifts significantly to the left under the bear market regime (dashed red line). This suggests that within the game-theoretic framework, the expected return of the risky asset is lower in a bear market and higher in a bull market.

These results offer a fascinating comparison to the classical, non-exploratory framework studied by \cite{savku2022stochastic}. In their research, the optimal strategies are purely deterministic, yielding $y^*=0$ for the investor and $z^*=r(\theta_s)$ for the market. Our numerical distributions align closely with these theoretical findings, as the highest density regions of our graphs hover around these classical point masses. However, they are not strictly equal. This expected divergence is precisely due to the introduction of entropy regularization in our framework. The entropy terms intrinsically prevent the policies from collapsing into deterministic Dirac measures, promoting a broader exploration of the continuous action spaces $\mathcal{Y}$ and $\mathcal{Z}$ while still gravitating toward the classical optimal solutions.

This analysis explores the behavior of the value function $V(x, i)$ at $t=0.5$ under different economic regimes within the exploratory ZSSDG framework through numerical examples. As illustrated in Figure \ref{fig:Value function}, regardless of whether it is a Bull or Bear regime, the value function exhibits a monotonically increasing trend as the initial wealth $x$ increases. This intuitively validates the positive contribution of wealth accumulation to the expected terminal utility. More importantly, the figure clearly demonstrates that at any given initial wealth level, the value function in the Bull regime is consistently and significantly higher than that in the Bear regime (i.e., the blue solid line lies strictly above the red dashed line). This finding is consistent with financial and economic intuition: although the market, acting as the minimizing player (player $II$) in our ZSSDG framework, continuously applies adversarial controls to suppress the investor's utility, the superior potential asset returns and underlying market fundamentals in the Bull regime remain dominant, thereby yielding a higher expected utility for the investor. Simultaneously, the variation in the curve's slope across different wealth intervals, along with its overall smooth extension, reflects the stabilizing effect of the exploratory stochastic strategies¡ªfollowing the introduction of entropy regularization¡ªin balancing jump-diffusion risks and regime-switching shocks. This effectively characterizes the dynamic features of the game equilibrium across different market cycles.
\begin{figure}[h!]
	\centering
\includegraphics[width=4in]{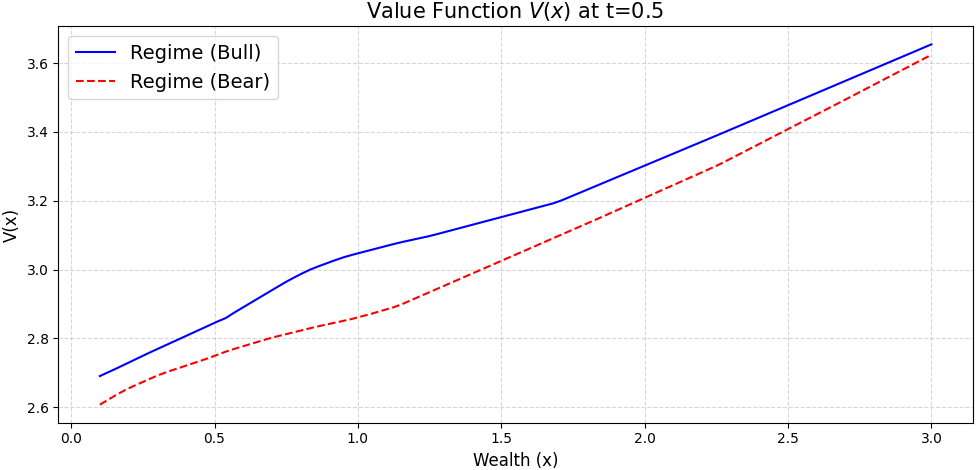}
	\caption{Value function $V(t,x)$.}
	\label{fig:Value function}
\end{figure}

To evaluate the training dynamics and the stability of our proposed Actor-Critic algorithm in solving the highly challenging exploratory SDG defined by \eqref{Investment_SDE} and \eqref{Investment_pay}, Figure \ref{fig:Network parameter} illustrates the $L_2$ norm evolution for the neural network parameters over $300$ episodes. The critic network parameters ($\psi$), responsible for approximating the highly complex exploratory value function $J(t, x, i; \boldsymbol{\pi}, \boldsymbol{\gamma})$, undergo an initial phase of significant adaptation before stabilizing around an $L_2$ norm of $8.3$. In contrast, the parameters for Actor $I$ ($\phi_1$, governing the investor's policy $\boldsymbol{\pi}_s$) and Actor $II$ ($\phi_2$, governing the market's policy $\boldsymbol{\gamma}_s$) remain remarkably stable throughout the entire training process. This bounded behavior in the parameter space demonstrates the robustness of the policy update mechanism, effectively preventing gradient explosion despite the complex jump-diffusion dynamics.
\begin{figure}[h!]
	\centering
\includegraphics[width=4in]{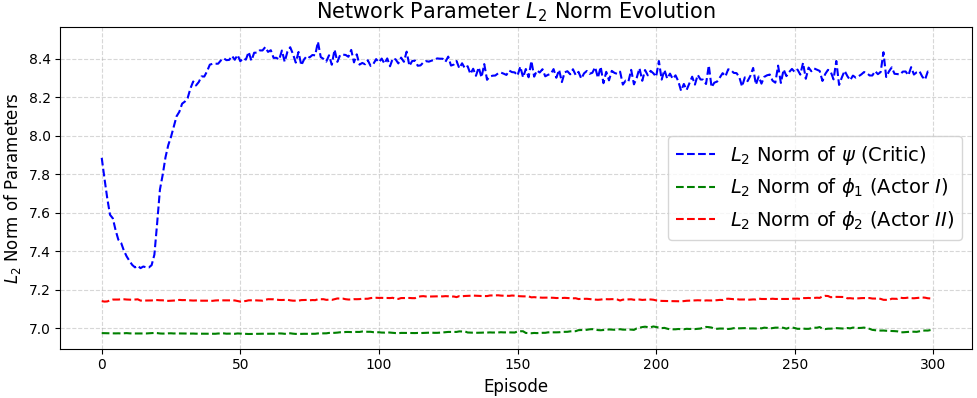}
	\caption{Network parameter $L_2$ norm evolution.}
	\label{fig:Network parameter}
\end{figure}
\begin{figure}[h!]
	\centering
\includegraphics[width=4in]{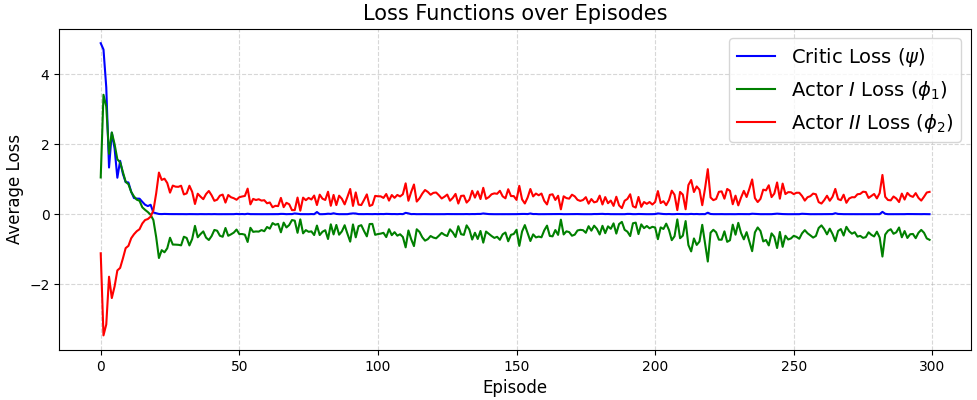}
	\caption{Loss function over episodes.}
	\label{fig:Loss function}
\end{figure}

Figure \ref{fig:Loss function} further confirms the algorithm's convergence by plotting the average loss functions over the training episodes. The critic's loss drops rapidly to near-zero within the first $25$ episodes, indicating that the network swiftly and accurately learns to evaluate the temporal-difference error. Crucially, the loss trajectories of Actor $I$ and Actor $II$ exhibit a distinct, mirror-like oscillation during the early stages of training. This symmetric fluctuation directly reflects the minimax nature of the zero-sum game: the investor seeks to maximize the expected utility while the market seeks to minimize it. As training progresses, these adversarial oscillations gradually dampen and converge toward a narrow band near zero, indicating that the parameterized policies $\boldsymbol{\pi}_s$ and $\boldsymbol{\gamma}_s$ have successfully reached a Nash equilibrium for the exploratory game.
\section{Conclusions}\label{Section 6}
In this paper, we develop an exploratory RL framework for ZSSDGs under a regime-switching jump-diffusion model. By incorporating entropy regularization, we derive the expressions for the exploratory SDEs and the payoff function, introducing a distributional control mechanism to effectively address inherent uncertainties, estimation errors, and structural shifts. We first apply this framework theoretically to a regime-switching linear-quadratic ZSSDG. To address problems in general settings, we design an Actor-Critic RL algorithm and utilize it to solve a regime-switching investment game.

Several promising directions for future research remain open, including investigating non-zero-sum SDGs with regime switching; exploring SDGs under constraints; analyzing the convergence between exploratory and classical value functions in general settings; extending the framework to cooperative games; and exploring broader engineering and financial applications in risk management and control systems.

\begin{appendices}
\section{Proof of Proposition \ref{Proposition 1}}\label{Appendix B}
\setcounter{equation}{0}
\renewcommand{\theequation}{B.\arabic{equation}}
\begin{proof}
\textbf{1. Local Lipschitz continuity for $\tilde{b}$ and $\tilde{\sigma}$:}
Under Assumption \ref{Assumption 1}, for all $t$, $x\in\mathbb{R}_K^d$, $x^\prime\in\mathbb{R}_K^d$, $i \in \mathcal{S}$, and admissible controls, we have
\begin{equation*}
\begin{aligned}
&\tilde{b}(t, x, i; \boldsymbol{\pi}(\cdot\mid t,x,i),\boldsymbol{\gamma}(\cdot\mid t,x,i)) - \tilde{b}(t, x^\prime, i; \boldsymbol{\pi}(\cdot\mid t,x^\prime,i),\boldsymbol{\gamma}(\cdot\mid t,x^\prime,i))\\
=&\int_{\mathcal{Y}} \left[\int_{\mathcal{Z}}  \Big(b(t, x, i; y, z) - b(t, x'; i; y, z)\Big) \boldsymbol{\gamma}(z\mid t,x,i)\,dz \right]\boldsymbol{\pi}(y\mid t,x,i)\,dy\\
&+\int_{\mathcal{Y}} \left[\int_{\mathcal{Z}} b(t, x^{\prime}, i; y, z) \left(\boldsymbol{\gamma}(z\mid t,x,i) - \boldsymbol{\gamma}(z\mid t,x^{\prime},i)\right)\,dz\right]\boldsymbol{\pi}(y\mid t,x,i) \, dy\\
&+ \int_{\mathcal{Y}}\left[ \int_{\mathcal{Z}} b(t, x^{\prime}, i; y, z) \boldsymbol{\gamma}(z\mid t,x^{\prime},i)\,dz\right] \left(\boldsymbol{\pi}(y\mid t,x,i) - \boldsymbol{\pi}(y\mid t,x^{\prime},i)\right)\,dy\\
\leq&  \int_{\mathcal{Y}}\left[ \int_{\mathcal{Z}}  \big|b(t, x, i; y, z) - b(t, x^{\prime}, i; y, z)\big| \boldsymbol{\gamma}(z\mid t,x,i)\, dz\right]\boldsymbol{\pi}(y\mid t,x,i) \, dy\\
&+ \int_{\mathcal{Y}} \left[\int_{\mathcal{Z}} b(t, x^{\prime}, i; y, z) \big|\boldsymbol{\gamma}(z\mid t,x,i) - \boldsymbol{\gamma}(z\mid t,x^{\prime},i)\big|\,dz\right]\boldsymbol{\pi}(y\mid t,x,i) \, dy\\
&+\int_{\mathcal{Y}}\left[ \int_{\mathcal{Z}} b(t, x^{\prime}, i; y, z) \boldsymbol{\gamma}(z\mid t,x^{\prime},i) \,dz\right] \big|\boldsymbol{\pi}(y\mid t,x,i) - \boldsymbol{\pi}(y\mid t,x^{\prime},i)\big|\,dy\\
\leq & C_K \big| x - x^{\prime} \big|+ M_b L_{\boldsymbol{\gamma}} \big| x - x^{\prime} \big| +M_b L_{\boldsymbol{\pi}} \big| x - x^{\prime} \big|
\leq  C \big| x - x^{\prime} \big|.
\end{aligned}
\end{equation*}
Similarly, we can prove that $\tilde{\sigma}(t, x, i;\boldsymbol{\pi}(\cdot\mid t,x,i),\boldsymbol{\gamma}(\cdot\mid t,x,i))$ is Local Lipschitz continuous for each $i \in \mathcal{S}$.

\textbf{2. Linear growth in $x$ for $\tilde{b}$ and $\tilde{\sigma}$:}
Under Assumption \ref{Assumption 1}, for all $t$, $x$, $i \in \mathcal{S}$, and admissible controls, we have
$$\begin{aligned}
&\big| \tilde{b}(t, x, i; \boldsymbol{\pi}(\cdot\mid t,x,i),\boldsymbol{\gamma}(\cdot\mid t,x,i)) \big|
\\
\leq& \int_{\mathcal{Y}} \left[\int_{\mathcal{Z}} \big|b(t, x, i; y, z)\big| \boldsymbol{\gamma}(z\mid t,x,i)\,dz \right]\boldsymbol{\pi}(y\mid t,x,i)\,dy\\
\leq& \int_{\mathcal{Y}}\left[ \int_{\mathcal{Z}} (C_1 \big| x \big| + C_2)\boldsymbol{\gamma}(z\mid t,x,i)\,dz\right]\boldsymbol{\pi}(y\mid t,x,i)\,dy
= C_1 \big| x \big| + C_2.
\end{aligned}$$
Similarly, we can prove that $\tilde{\sigma}(t, x, i; \boldsymbol{\pi}_{t, x, i},\boldsymbol{\gamma}_{t, x, i})$ has linear growth in $x$.

\textbf{3. Local Lipschitz continuity for jump:}
Under Assumption \ref{Assumption 1} and \ref{Assumption 2}, for all $t$, $x\in\mathbb{R}_K^d$, $x^\prime\in\mathbb{R}_K^d$, $i \in \mathcal{S}$, and admissible controls, using the triangle inequality, we observe that for each $k = 1, \cdots, \ell$,
$$\begin{aligned}
&|\eta_k(t, x, i;  G_{\boldsymbol{\pi}}(t,x,i,u),  M_{\boldsymbol{\gamma}}(t,x,i,v); w)\\
&-\eta_k(t, x^{\prime}, i; G_{\boldsymbol{\pi}}(t,x^{\prime},i,u),  M_{\boldsymbol{\gamma}}(t,x^{\prime},i,v); w)|\\
\leq &| \eta_k(t, x, i;  G_{\boldsymbol{\pi}}(t,x,i,u),  M_{\boldsymbol{\gamma}}(t,x,i,v); w)\label{.1}\\
&\quad- \eta_k(t, x, i;  G_{\boldsymbol{\pi}}(t,x^{\prime},i,u),  M_{\boldsymbol{\gamma}}(t,x^{\prime},i,v); w) | \quad \text{(Te 1)}\\
&+ | \eta_k(t, x, i;  G_{\boldsymbol{\pi}}(t,x^{\prime},i,u),  M_{\boldsymbol{\gamma}}(t,x^{\prime},i,v); w)\\
&\quad- \eta_k(t, x^{\prime}, i;  G_{\boldsymbol{\pi}}(t,x^{\prime},i,u),  M_{\boldsymbol{\gamma}}(t,x^{\prime},i,v); w) | \quad \text{(Te 2)}.
\end{aligned}$$
From Assumption \ref{Assumption 2}, we can derive the property of (Te 1):
$$\begin{aligned}
&|\eta_k(t, x, i;  G_{\boldsymbol{\pi}}(t,x^{\prime},i,u),  M_{\boldsymbol{\gamma}}(t,x^{\prime},i,v); w)\\
&- \eta_k(t, x, i;  G_{\boldsymbol{\pi}}(t,x,i,u),  M_{\boldsymbol{\gamma}}(t,x,i,v); w)|\\
\leq& L_{\eta_k} (\left| G_{\boldsymbol{\pi}}(t,x^{\prime},i,u) -  G_{\boldsymbol{\pi}}(t,x,i,u)\right|\\
&+ \left| M_{\boldsymbol{\gamma}}(t,x^{\prime},i,v) -  M_{\boldsymbol{\gamma}}(t,x,i,v)\right|)
\leq L_{\eta_k} (L_G + L_M) \left|x - x'\right|.
\end{aligned}$$
From Assumption \ref{Assumption 1}, we can derive the property of (Te 2):
$$\begin{aligned}
&|\eta_k(t, x, i;  G_{\boldsymbol{\pi}}(t,x^{\prime},i,u),  M_{\boldsymbol{\gamma}}(t,x^{\prime},i,v); w)\\
&- \eta_k(t, x^{\prime}, i;  G_{\boldsymbol{\pi}}(t,x^{\prime},i,u),  M_{\boldsymbol{\gamma}}(t,x^{\prime},i,v); w)|
\leq L_{x}\left|x-x^{\prime}\right|.
\end{aligned}$$
Finally, we have
$$\begin{aligned}
&\int_{\mathbb{R}^+ \times[0,1]^{n1}\times[0,1]^{n2}}\left|\eta_k(t, x, i;  G_{\boldsymbol{\pi}}(t,x,i,u),  M_{\boldsymbol{\gamma}}(t,x,i,v); w)\right.\\
&\left.~-\eta_k(t, x^{\prime}, i; G_{\boldsymbol{\pi}}(t,x^{\prime},i,u),  M_{\boldsymbol{\gamma}}(t,x^{\prime},i,v); w)\right|^p\nu_k(d w) dudv\\
\leq&\int_{\mathbb{R}^+ \times[0,1]^{n1}\times[0,1]^{n2}}L\left| x-x'\right|^p\nu_k(d w) dudv
\leq L\left| x-x'\right|^p.
\end{aligned}$$

\textbf{4. Linear growth in $x$ for jump:}
Under Assumption \ref{Assumption 1}, for all $t$, $x$, $i \in \mathcal{S}$, and admissible controls, we observe that for each $k = 1, \cdots, \ell$,
$$\begin{aligned}
&\int_{\mathbb{R}^+\times[0,1]^{n1}\times[0,1]^{n2}}
\quad\left|\eta_k\left(t, x, i; G_{\boldsymbol{\pi}}(t,x,i,u), M_{\boldsymbol{\gamma}}(t,x,i,v); w\right)\right| \nu_k(dw) dudv\\
=&\int_{\mathcal{Y}}\left[\int_{\mathcal{Z}} \int_{\mathbb{R}^+} \big| \eta_k(t, x, i; y, z;w) \big| \nu_k(dw) \boldsymbol{\gamma}(z\mid t,x,i)\, dz\right]\boldsymbol{\pi}(y\mid t,x,i)\, dy.
\end{aligned}$$
From Assumption \ref{Assumption 1}, we have
$$\sum_{k=1}^\ell \int_{\mathbb{R}^+} \big| \eta_k(t, x, i; y, z;w) \big|\nu_k(dw) \leq C_1+C_2| x |.$$
Combining \textbf{1, 2, 3, 4} and applying Theorem 119 in \cite{situ2005theory} (along with the standard extension for piecewise constant c\`{a}dl\`{a}g processes to accommodate the finite-state continuous-time Markov chain $\alpha_t$), we can establish the existence and uniqueness of the strong solution to the regime-switching exploratory SDE.
\end{proof}




\end{appendices}


\bibliography{sn-bibliography}

\end{document}